\documentclass{article}
\usepackage[final]{neurips_2021}
\usepackage{hyperref}
\usepackage[utf8]{inputenc}
\usepackage{amsmath,amsthm,amssymb}
\usepackage{xcolor}
\usepackage{enumitem}
\usepackage{url}

\usepackage{algorithm}
\usepackage{algorithmic}
\usepackage{thm-restate}
\usepackage{balance}
\usepackage{bbm}
\usepackage[mathscr]{euscript}

\setlist{nosep}

\newcommand{\norm}[1]{\left\Vert #1 \right\Vert}

\newcommand{\R}{\mathbb{R}}
\newcommand{\B}{\mathbb{B}}
\newcommand{\E}{\mathbb{E}}

\newcommand{\iprod}[1]{\langle #1 \rangle}

\newcommand{\cB}{\mathcal{B}}
\newcommand{\cA}{\mathcal{A}}

\newcommand{\argmin}{\mathop{\text{argmin}}}

\newcommand{\bad}{\mathcal{B}}
\newcommand{\ftrl}{\mathcal{A}_{\text{ftrl}}}
\newcommand{\ogd}{\mathcal{A}_\text{ogd}}
\newcommand{\less}{\mathcal{A}_\text{hints}}

\newcommand{\goodindex}{S}

\newcommand{\unconstrained}{\mathcal{A}_\text{unc}}
\newcommand{\unconstrainedoned}{\mathcal{A}_\text{unc-1D}}

\newtheorem{theorem}{Theorem}[section]
\newtheorem{lemma}[theorem]{Lemma}
\newtheorem{proposition}[theorem]{Proposition}

\newtheorem{defn}[theorem]{Definition}

\newcommand{\mnote}[1]{\textcolor{red}{[MP: #1]}}

\newcommand{\regret}{\mathcal{R}}
\newcommand{\cost}{\mathrm{cost}}
\newcommand{\query}{\mathcal{Q}}

\title{Logarithmic Regret from Sublinear Hints}

\author{
  Aditya Bhaskara \\
  Department of Computer Science\\
  University of Utah\\
  Salt Lake City, UT \\
  \texttt{bhaskaraaditya@gmail.com} \\
   \And
   Ashok Cutkosky \\
   Dept. of Electrical and Computer Engineering\\
   Boston University \\
   Boston, MA \\
   \texttt{ashok@cutkosky.com} \\
   \AND
   Ravi Kumar \\
   Google Research \\
   Mountain View, CA \\
   \texttt{ravi.k53@gmail.com} \\
   \And
   Manish Purohit \\
   Google Research \\
   Mountain View, CA \\
   \texttt{mpurohit@google.com} \\
}

\begin{document}

\maketitle

\begin{abstract}
    We consider the online linear optimization problem, where at every step the algorithm plays a point $x_t$ in the unit ball, and suffers loss $\iprod{c_t, x_t}$ for some cost vector $c_t$ that is then revealed to the algorithm. Recent work showed that if an algorithm  receives a {\em hint} $h_t$ that has non-trivial correlation with $c_t$ before it plays $x_t$, then it can achieve a regret guarantee of $O(\log T)$, improving on the bound of $\Theta(\sqrt{T})$ in the standard setting. In this work, we study the question of whether an algorithm really requires a hint at {\em every} time step. Somewhat surprisingly, we show that an algorithm can obtain $O(\log T)$ regret with just $O(\sqrt{T})$ hints under a natural query model; in contrast, we also show that $o(\sqrt{T})$ hints cannot guarantee better than $\Omega(\sqrt{T})$ regret. We give two applications of our result, to the well-studied setting of optimistic regret bounds and to the problem of online learning with abstention. 
\end{abstract}

\section{Introduction}

There has been a spate of work on improving the performance of online algorithms with the help of externally available hints.  The goal of these works is to circumvent worst-case bounds and exploit the capability of machine-learned models that can potentially provide these hints.  There have been two main lines of study.  The first is for combinatorial problems, where the goal has been to be improve the competitive ratio of online algorithms; problems considered here include ski-rental~\cite{gollapudi2019online,KPS18},  caching~\cite{jiang2020online,lykouris2018competitive,rohatgi2020near}, scheduling~\cite{BamasMRS20,im2021non,KPS18,mitzenmacher2020scheduling}, matching~\cite{KPSSV19,flow-prediction}, etc. The second is in the learning theory setting, where the goal has to been to improve the regret of online optimization algorithms.  A series of recent papers showed how to achieve better regret guarantees, assuming that we have a hint about the cost function before the algorithm makes a choice. For many variants of the online convex optimization problem, works such as~\cite{rakhlin2013online, dekel2017online, olo-hints, mhints} studied the power of having prior information about a cost function. Works such as~\cite{WeiLuo} have also studied improved regret bounds in partial information or bandit settings. In all these works, a desirable property is to ensure \emph{consistency}, which demands better performance with better quality hints, and \emph{robustness}, which guarantees a certain level of performance with poor quality or even adversarially bad hints.

Recall the standard online optimization model~\cite{zinkevich2003online}, which is a game between an algorithm and an adversary.  In each round, the algorithm plays a point and the adversary responds with a cost function that is visible to the algorithm, and the cost in this round is measured by evaluating the cost function on the played point. In online linear optimization, the cost function is linear.  The regret of an algorithm is the worst-case difference between the cost of the algorithm and the cost of an algorithm that plays a fixed point in each round.  A natural way to incorporate hints or prior information is to give the algorithm access to a hint about the cost function at a given round \emph{before} it chooses a point to play at that round. In this sense,  hints are present \emph{gratis}~\cite{rakhlin2013online}.

The availability of a hint in each round seems natural in some settings, e.g., when cost functions change gradually with time~\cite{rakhlin2013online, bubeck19b}). However, it can be prohibitive in many others, for instance, if hints are obtained by using expensive side information, or if they are generated by a running a computationally expensive ML model.  Furthermore, hints can also be wasteful when the problem instance, or even a large sub-instance, is such that the algorithm cannot really derive any substantial benefit from their presence.  This leads to the question of strengthening the online learning with hints model by making it parsimonious.  In this work, we pursue this direction, where we offer the algorithm the flexibility to \emph{choose} to ask for a hint before it plays the point.  It then becomes an onus on the algorithm to know when to ask for a hint and how to use it judiciously, while ensuring both  consistency and robustness.  The performance of such an algorithm is measured not only by its regret but also by the number of hints it uses.  

\vspace{-6pt}
\paragraph{Our contributions.}
We now present a high level summary of our results. For ease of exposition, we will defer the formal statements to the respective sections. All our results are for the problem of online {\em linear} optimization (OLO) when the domain is the unit $\ell_2$ ball; the cost function at every step is defined using a cost vector $c_t$ (and the cost or loss is the inner product of the point played with $c_t$). We call a hint \emph{perfect} if it is the same as the cost vector at that time step, \emph{good} if it is weakly correlated with the cost vector, and \emph{bad} otherwise.  %All the results below are  informally stated for sake of exposition; 

As our main result, we show that for OLO in which hints are guaranteed to be good whenever the algorithm asks for a hint, there is an efficient randomized algorithm that obtains $O(\log T)$ regret using only $O(\sqrt{T})$ hints. We extend our result to the case when $|\cB|$ hints can be bad (chosen in an oblivious manner, as we will discuss later), and give an algorithm that achieves a regret bound of $O(\sqrt{|\cB|} \log T)$, while still asking for $O(\sqrt{T})$ hints.
It is interesting to contrast our result with prior work: Dekel et al.~\cite{dekel2017online} obtained an algorithm with $O(\log T)$ regret, when a good hint is available in every round.  Bhaskara et al.~\cite{olo-hints} made this result robust, obtaining a regret bound of $O(\sqrt{|\cB|} \log T)$, when there are $|\cB|$ bad hints.  Our result improves upon these works by showing the same asymptotic regret bounds, but using only $O(\sqrt{T})$ hints.  Our result also has implications for optimistic regret bounds~\cite{rakhlin2013online} (where we obtain the same results, but with fewer hints) and for online learning with abstention~\cite{neu2020fast} (where we can bound the number of abstentions).
 
We also show two lower bounds that show the optimality of our algorithm.  The first is regarding the minimum number of hints needed to get $O(\log T)$ regret: we show that any (potentially randomized) algorithm that uses $o(\sqrt{T})$ hints will suffer a regret of $\Omega(\sqrt{T})$.  The second and more surprising result is the role of randomness: we show that any deterministic algorithm that obtains $O(\log T)$ regret must use $\Omega(T/\log(T))$ hints, even if each of them is perfect. This shows the significance of having a randomized algorithm and an oblivious adversary.

Finally we extend our results to the unconstrained OLO setting (see Section~\ref{sec:unconstrained}), where we design 
a  deterministic algorithm to obtain $O(\log^{3/2} T)$ regret (suitably defined for the unconstrained case) when all hints are good, and 
a randomized algorithm to obtain $O(\log T)$ regret, which can be extended to the presence of bad hints.

There are three aspects of our results that we find surprising.  The first is even the possibility of obtaining $O(\log T)$ regret using only a sublinear number of hints.  The second is the sharp threshold on the number of hints needed to obtain logarithmic regret; the regret does not gracefully degrade when the number of hints is below $O(\sqrt{T})$.  The third is the deterministic vs randomized separation between the constrained and the unconstrained cases, when all queried hints are good.

We now present some intuition why our result is plausible.  Consider the standard ``worst-case'' adversary for OLO: random mean-zero costs. This case is ``hard'' because the learner achieves zero expected cost, but the competitor achieves $-\sqrt{T}$ total cost. However, if we simply play a hint $-h_t$ on $O(\sqrt{T}/\alpha)$ rounds, each such round incurs $-\alpha$ cost, which is enough to cancel out the $\sqrt{T}$ regret while making only $O(\sqrt{T})$ hint queries. Thus, the standard  worst-case instances are actually \emph{easy} with hints. More details on this example and an outline of our algorithm are provided at the start of Section~\ref{sec:alg}. %Thus, the standard  worst-case instances are actually \emph{easy} with hints. To leverage this, we tighten standard FTRL analysis to show that it already achieves logarithmic regret on ``obviously non-random'' losses. We couple this with an external algorithm that monitors the performance of the FTRL subroutine, and requests a hint when it appears to be exhibiting worst-case behavior.

\paragraph{Organization.}
Section~\ref{sec:prelim} provides the necessary background.  The main algorithm and analysis for the constrained OLO case are in Section~\ref{sec:alg}.  The extensions and applications of this can be found in Section~\ref{sec:ext}.  Section~\ref{sec:lb} contains the lower bounds and Section~\ref{sec:unconstrained} contains the algorithms and analyses for the unconstrained case.  All missing proofs are in the Supplementary Material.  
\section{Preliminaries}
\label{sec:prelim}

Let $\|\cdot\|$ denote the $\ell_2$-norm and $\B^d = \{ x \in \R^d \mid \|x\| \leq 1\}$ denote the unit $\ell_2$ ball in $\R^d$.  We use the compressed sum notation and use $c_{1:t}$ to denote $\sum_{i=1}^t c_i$ and $\|c\|^2_{1:t}$ to denote $\sum_{i=1}^t \|c_i\|^2$.  Let $\vec{c} = c_1, \ldots, c_T$ be a sequence of cost vectors.  Let $[T] = \{1, \ldots, T\}$.

{\bf OLO problem and Regret.} The \emph{constrained online linear optimization} (OLO) problem is modeled as a game over $T$ rounds. At each time $t \in [T]$, an algorithm $\cA$ plays a vector $x_t \in \B^d$, and then an adversary responds with a \emph{cost vector} $c_t \in \B^d$. The algorithm incurs cost (or loss) $\iprod{x_t, c_t}$ at time $t$. The total cost incurred by the algorithm is $\cost_{\cA}(\vec{c}) = \sum_{t=1}^T \iprod{x_t, c_t}$. The regret of the algorithm $\cA$ with respect to a `comparator' or benchmark vector $u \in \B^d$ is 
% In the (constrained) online linear optimization (OLO) problem, at each time $t \in [T]$: (i) an algorithm $\cA$ chooses a vector $x_t \in \B^d$; (ii) the adversary responds with a \emph{cost vector} $c_t \in \B^d$; and (iii) the algorithm $\cA$ incurs  cost $\iprod{x_t, c_t}$ at time $t$.  The \emph{total cost} of the algorithm is $\cost_{\cA}(\vec{c}) = \sum_{t = 1}^T \iprod{x_t, c_t}$.  The regret of the algorithm $\cA$ with respect to $u \in \B^d$ is 
\[
\regret_{\cA}(u, \vec{c}) = \cost_{\cA}(\vec{c}) - \cost_{\cA_u}(\vec{c}) = \sum_{t = 1}^T \iprod{x_t - u, c_t},
\]
where $\cA_u$ is the algorithm that always plays $u$ at every time step.  
The \emph{regret} of an algorithm $\cA$ is its worst-case regret with respect to all $u \in \B^d$:
\[
\regret_{\cA}(\vec{c}) = \sup_{u \in \B^d} \regret_{\cA}(u, \vec{c}).
\]

{\bf Hints and query cost.} Let $\alpha > 0$ be fixed and known. In this paper we consider the OLO setting where, at any round $t$ before choosing $x_t$, an algorithm $\cA$ is allowed to obtain a \emph{hint} $h_t \in \B^d$. If $\langle h_t, c_t\rangle \ge \alpha \|c_t\|^2$, we say that the hint is \emph{$\alpha$-good}. 
If $\cA$ opts to obtain a hint at time $t$, then it incurs a query cost of $\alpha \|c_t\|^2$; the query cost is $0$ if no hint was obtained at time $t$.  The definition of regret stays the same and we denote it by $\regret_{\cA, \alpha}(\cdot)$.  The total \emph{query cost} of $\cA$ is given by $\query_{\cA, \alpha}(\vec{c}) = \sum_{t = 1}^T \mathbbm{1}_{t} \cdot \alpha \|c_t\|^2$ where $\mathbbm{1}_{t}$ is an indicator function used to denote whether $\cA$ queried for a hint at time $t$. Note that the algorithm does not actually know the query cost for a round until the end of the round.
If $\cA$ is a randomized algorithm, the notions of expected regret and expected query cost follow naturally. We consider the setting when the adversarial choice of the hint $h_t$ and cost vector $c_t$ at time $t$ is oblivious to whether the algorithm queries for a hint at time $t$ but can depend adaptively on all previous decisions.

More generally, we consider the case that some subset of the hints are ``bad'' in the sense that $\langle c_t, h_t\rangle<\alpha \|c_t\|^2$; we let $\bad$ denote the set of such indices $t$. Although we assume $\alpha$ is known to our algorithms, we do \emph{not} assume any information about $\bad$. Further, our algorithm is charged $\alpha \|c_t\|^2$ for querying a hint \emph{even} if the hint was bad.

\section{Main algorithm}
\label{sec:alg}

\paragraph{Intuition and outline.} The high-level intuition behind our algorithm is the following: suppose for a moment that $\alpha = 1/2$ and each hint is $\alpha$-correlated with the corresponding cost. Now suppose the cost vectors $c_1, \dots, c_T$ are random unit vectors, as in the standard tight example for FTRL. In this case, if an algorithm were to make a hint query for the first $4 \sqrt{T}$ steps, set $x_t = -h_t$ in those steps, and play FTRL subsequently, then the cost incurred by the algorithm will be less than $-2\sqrt{T}$  in the first $4\sqrt{T}$ steps, and $0$ (in expectation) subsequently. On the other hand, for random vectors, we have $\norm{c_{1:T}} \le 2\sqrt{T}$ with high probability, and thus the best vector in hindsight achieves a total cost $- \norm{c_{1:T}} \ge -2\sqrt{T}$.  Thus the algorithm above actually incurs regret $\le 0$.

It turns out that the key to the above argument is $\norm{c_{1:T}}$ being small. In fact, suppose that $c_t$ are unit vectors, and assume that $\norm{c_{1:T}} \le T/4$. Now, suppose the algorithm makes a hint query at $10\sqrt{T}$ random indices, sets $x_t = -h_t$ in those steps, and uses FTRL in the other steps. One can show that the cost incurred by the algorithm is $-5 \sqrt{T}$ plus the cost of the FTRL steps. Since the cost in the FTRL steps is within $\sqrt{T}$ of the cost incurred by the competitor $u$, we can show that the regret is once again $\le 0$. The missing subtlety here is accounting for the cost of $u$ on the query steps, but using the bound on $\norm{c_{1:T}}$, this can be adequately controlled if the queries are done at random.

The above outline suggests that the difficult case is $\norm{c_{1:T}}$ being {\em large}. However, it turns out that prior work of Huang et al.~\cite{HuangEtAl} showed that when the domain is the unit ball, FTL achieves logarithmic regret if we have $\norm{c_{1:t}} \ge \Omega(t)$ for all $t$. 

Our algorithm can be viewed as a combination of the two ideas above. If we could identify the largest round $S$ for which $\norm{c_{1:S}}$ is small, then we can perform uniform sampling until round $S$ and use FTL subsequently, and hope that we have logarithmic regret overall (although it is not obvious that the two bounds {\em compose}). The problem with this is that we do not know the value of $S$. We thus view the problem of picking a sampling probability as a one-dimensional OLO problem in itself, and show that an online gradient descent (OGD) algorithm achieves low regret when compared to all sequences that have the structure of being $\delta$ until a certain time and $0$ thereafter (which captures the setting above).  The overall algorithm thus picks the querying probability using the OGD procedure, and otherwise resorts to FTRL, as in the outline above. 

We now present the details of the overall algorithm (Section~\ref{sec:full-alg}) and the two main subroutines it uses (Sections~\ref{sec:passiveftrl} and \ref{sec:dynamic}).

%At a high level, our algorithm operates as follows.

%As discussed earlier, we formulate the problem of finding the hint-querying probability itself as online learning in one dimension. 

\subsection{A sharper analysis of FTRL}
\label{sec:passiveftrl}

In this section we consider the classic adaptive \emph{Follow the Regularized Leader} (FTRL) algorithm $\ftrl$ (Algorithm~\ref{alg:ftrl}) and show a regret bound that is better than the usual one, if the length of the aggregate cost vector grows ``rapidly'' after a certain initial period. 

%
% our goal is to design an online learning algorithm that has the following property: if $S$ is the largest index such that $\|c_{1:S}\|\le \frac{\alpha}{4}(1+\|c\|^2_{1:S})$, then for all $\|u\|\le 1$:
% \[
% \regret_{\cA}(u, \vec{c}) \leq O(\sqrt{1+\|c\|^2_{1:S}}) + O\left( \frac{\log(T)}{\alpha}\right) 
% \]
% \begin{enumerate}
%     \item $\sum_{t=1}^B \langle c_t, x_t-u\rangle \le O(\sqrt{1+\|c\|^2_{1:B}})$.
%     \item $\sum_{t=1}^T \langle c_t, x_t-u\rangle \le \sum_{t=1}^B \langle c_t, x_t - x_{B+1}\rangle + O\left( \frac{\log(T)}{\alpha}\right)$.
% \end{enumerate}
%
%In fact, we show that ordinary adaptive FTRL attains the above regret guarantee. 
For convenience, let $\sigma_t  = \norm{c_t}^2$.  Define the regularizer terms as $r_0 = 1$ and for $t \ge 1$, let
\begin{equation}  
r_t = \sqrt{1 + \sigma_{1:t}} - \sqrt{1 + \sigma_{1:t-1}}.
\label{eq:rt-definition}
\end{equation}
By definition, we have $r_{0:t} = \sqrt{1 + \sigma_{1:t}}$. Furthermore, we have $r_t < 1$ for all $t$, since $\sigma_t = \norm{c_t}^2 \le 1$.  The FTRL algorithm $\ftrl$ then plays the points $x_1, x_2, \dots$, which are defined as 
\begin{align}
    x_{t+1} &= \argmin_{\|x\|\le 1} \left\{ \langle c_{1:t},x\rangle  + \frac{r_{0:t}}{2}\|x\|^2 \right\}. \label{eq:ftrl-iterate}
\end{align}

%
%Note that the regularization parameters satsify:
%\begin{align*}
%    r_{0:t}&= \sqrt{1+\|c\|^2_{1:t}}\\
%    r_t &=\sqrt{1+\|c\|^2_{1:t}}-\sqrt{1+\|c\|^2_{1:t-1}}\le 1
%\end{align*}

\noindent
\begin{minipage}{\textwidth}
%\begin{wrapfigure}{L}{0.55\textwidth}
\begin{minipage}[t]{0.5\textwidth}
\begin{algorithm}[H]
   \caption{Adaptive FTRL $\ftrl$.}
   \label{alg:ftrl}
   \begin{algorithmic}
    \STATE $x_1 \leftarrow 0, r_0 \leftarrow 1$
    \FOR{$t=1,\dots,T$}
    	\STATE Play point $x_t$
    	\STATE Receive cost $c_t$
    	\STATE $r_{0:t} \leftarrow \sqrt{1 + \|c\|^2_{1:t}}$% - \sqrt{1 + \|c\|^2_{1:t-1}}$
    	\STATE $x_{t+1} \leftarrow \displaystyle{\argmin_{\|x\|\le 1}} \left\{ \langle c_{1:t},x\rangle  + \frac{r_{0:t}}{2}\|x\|^2 \right\}$
	\ENDFOR
\end{algorithmic}
\end{algorithm}
\end{minipage}
\begin{minipage}[t]{0.5\textwidth}
\begin{algorithm}[H]
   \caption{OGD with shrinking domain $\ogd$.}
   \label{alg:ogd}
   \begin{algorithmic}
   \REQUIRE Parameter $\lambda$
    \STATE $p_1 \gets 0, D_1 \gets [0, 1]$
    % \STATE Define loss function
    % \begin{align*}
    %     \ell(z, a) = \left\{\begin{array}{lr}
    % 0& z>a\\
    % \frac{1}{a}(a-z)^2& z\in[0,a]\\
    % a - 2z& z<0\end{array}\right.
    % \end{align*}
    \FOR{$t=1,\dots,T$}
    	\STATE Play point $p_t$
    	\STATE Receive cost $z_t$ and $\sigma_t\le 1$ 
    	\STATE \COMMENT{$\sigma_t$ will eventually be set to $\|c_t\|^2$}
    	\STATE $D_t \gets [0, \min(1, \frac{\lambda}{\sqrt{1+\sigma_{1:t}}})]$
    	\STATE $\eta_t \gets \frac{\lambda}{1 + \sigma_{1:t}}$
    	\STATE $p_{t+1} \gets \Pi_{D_t} \left( p_t - \eta_t z_t \right)$, where $\Pi_{D_t}$ is the projection to the interval
	\ENDFOR
\end{algorithmic}
\end{algorithm}
\end{minipage}
\end{minipage}

We will show that $\ftrl$ satisfies the following refined regret guarantee:

\begin{theorem}\label{thm:ftrlregret}
Consider $\ftrl$ on a sequence $c_t$ of cost vectors and let $\alpha \in (0, 1)$ be any parameter.  Suppose that $S$ is an index such that for all $t > S$, $\norm{c_{1:t}} \ge \frac{\alpha}{4} (1 + \sigma_{1:t})$ (recall $\sigma_t=\|c_t\|^2$).  Then,
\begin{enumerate}
\item For all $N \in [T]$ and for any $\norm{u} \le 1$, we have 
\[ \sum_{t =1}^N \iprod{c_t, x_t - u} \le 4.5 \sqrt{1 + \sigma_{1:N}}.  \]
\item For the index $S$ defined above, we have the refined regret bound:  
\begin{equation}\label{eq:ftrl-to-show}	
\regret_{\ftrl}(\vec{c})
\le \frac{\sqrt{1 + \sigma_{1:\goodindex}}}{2} + \frac{18 + 8 \log(1 + \sigma_{1:T})}{\alpha} + \|c_{1:\goodindex}\| + \sum_{t=1}^\goodindex \iprod{c_t, x_t}.
\end{equation}%\anote{should there not still be a 5 coefficient on the first term?}
\end{enumerate}
\end{theorem}

Part (1) of the theorem follows from the standard analysis of FTRL; we include the proof in Appendix~\ref{sec:app:ftrlaux} for completeness.  Part (2) is a novel contribution, where we show that if $\norm{c_{1:t}}$ grows quickly enough, then the ``subsequent'' regret is small. It can be viewed as a generalization of a result of Huang et al.~\cite{HuangEtAl}, who proved such a regret bound for $S = 0$.

\subsection{Switch-once dynamic regret}\label{sec:dynamic}

In this section we show a regret bound against all time-varying comparators of a certain form.  More formally, we consider the one-dimensional OLO problem where the costs are $z_t$ and $\lambda \geq 1$ is a known parameter. We also assume that at time $t$, the algorithm is provided with a parameter $\sigma_t\in [0,1]$ that will give some extra information about the sequence of comparators of interest. Thus the modified OLO game can be described as follows:

\begin{itemize}
    \item For $t= 1, 2, \dots$, the algorithm first plays $p_t \in [0,1]$, and then $z_t, \sigma_t$ are revealed.
    \item $z_t$ always satisfies $z_t^2 \le 4 \sigma_t$.
    \item We wish to minimize the regret with respect to a class of comparator sequences $(q_t)_{t=1}^T$ (defined below), i.e., minimize $\sum_t z_t (p_t - q_t)$ over all sequences in the class.
\end{itemize}

We remark that for the purposes of this subsection, we can think of $\sigma_t$ as $z_t^2$.  (The more general setup is needed when we use this in
Algorithm~\ref{alg:lesshints}.)

%We assume at time $t$ times that the comparators lie in $[0,\lambda/(1+\sigma_{1:t})]$. Of course, our probabilities will need to be related to the original costs $c_t$ eventually. To this end, we will eventually set $\sigma_t = \|c_t\|^2$ and will consider exclusively costs $z_t$ that satisfy $z_t^2\le 4\sigma_t$. %Let us define $\sigma_t = z_t^2$ for convenience. 

\begin{defn}[Valid-in-hindsight sequences]\label{def:valid-seq}
We say that a sequence $(q_t)_{t=1}^T$ is {\em valid in hindsight} if there exists an $S \in [T]$ and a $\delta \in [0,1]$ such that
\begin{enumerate}
\item $q_t = \delta$ for all $t \le S$ and $q_t = 0$ for $t > S$.
\item At the switching point, we have $\delta^2 \le \frac{\lambda^2}{1 + \sigma_{1:S}}$.
\end{enumerate}
\end{defn}

We now show that a variant of online gradient descent (OGD) with a {\em shrinking domain} achieves  low regret with respect to all valid-in-hindsight sequences; we call this $\ogd$ (see Algorithm~\ref{alg:ogd}).

\begin{restatable}{theorem}{thmogd}\label{thm:ogd}
Let $\lambda\ge 1$ be a given parameter, and $(z_t)_{t=1}^T$ be any sequence of cost values satisfying $z_t^2\le 4\sigma_t$. Let $(q_t)_{t=1}^T$ be a valid-in-hindsight sequence. The points $p_t$ produced by $\ogd$ then satisfy:
\[ \sum_{t=1}^T z_t (p_t - q_t) \le   \lambda \left( 1+ 3 \log (1 + \sigma_{1:T}) \right).    \]
\end{restatable}

\subsection{Full algorithm}\label{sec:full-alg}

In this section we present the full algorithm that utilizes the above two ingredients.

\begin{algorithm}[H]
   \caption{Algorithm with hints $\less$.}
   \label{alg:lesshints}
   \begin{algorithmic}
   \REQUIRE Parameter $\alpha$
   \STATE Initialize an instance of $\ftrl$ and an instance of $\ogd$ with parameter $\lambda=10/\alpha$
    \FOR{$t=1,\dots,T$}
        \STATE Receive $p_t$ from $\ogd$; Receive $x_t$ from $\ftrl$
        \STATE With probability $p_t$, get a hint $h_t$ and play $\hat x_t = -h_t$; otherwise, play $\hat x_t = x_t$
        \STATE Receive $c_t$
        \STATE Send $c_t$ to $\ftrl$ as $t$th cost; Send $z_t=-\alpha\|c_t\|^2 - \langle c_t,x_t\rangle$ and $\sigma_t=\|c_t\|^2$  to $\ogd$
	\ENDFOR
\end{algorithmic}
\end{algorithm}

% Our algorithm outline is the following:
% \begin{enumerate}
%     \item On round $t$, get a hint with probability $p_t$.
%     \item If we got a hint $h_t$, just play $\hat x_t = -h_t$.
%     \item Otherwise, play $\hat x_t = x_t$ where $x_t$ comes from the passive ftrl algorithm from section \ref{sec:passiveftrl}.
%     \item See cost $c_t$, update passive ftrl algorithm, compute new $p_{t+1}$ using gradient descent algo from section \ref{sec:dynamic} using $z_t = -\alpha\|c_t\|^2 - \langle c_t, x_t\rangle$, $\sigma_t=\|c_t\|^2$ and $K=20$
% \end{enumerate}
\begin{theorem}
\label{thm:fullexpectedregret}
When $\bad=\emptyset$,
\[
    \E[\regret_{\less, \alpha}(\vec{c})] \le \frac{78+38\log(1+\|c\|^2_{1:T})}{\alpha}
    \mbox{ and }
\E[\query_{\less, \alpha}(\vec{c})] \le  20\sqrt{\|c\|^2_{1:T}}.
\]
% where $p_t$ is the probability that we query a hint on round $t$.
\end{theorem}

\begin{proof}
Let us first bound the expected cost of querying the hints. From the description of $\ogd$ (Algorithm~\ref{alg:ogd}), because of the shrinking domain, we have $p_t\le |D_{t-1}| \le \frac{10}{\alpha \sqrt{1+\sigma_{1:t-1}}}$.  At time $t$, the expected query cost paid by the algorithm is $p_t \cdot \alpha \norm{c_t}^2$.  Using the above, we can bound this as
\[
\sum_{t=1}^T  p_t \alpha \norm{c_t}^2 \le \sum_{t=1}^T \frac{10 \sigma_t}{\sqrt{1+\sigma_{1:t-1}}} \le  \sum_{t=1}^T \frac{10 \sigma_t}{\sqrt{\sigma_{1:t}}}\\
\le 20\sqrt{\sigma_{1:T}}= 20\sqrt{\|c\|^2_{1:T}},
\]
where the last inequality follows from concavity of the square root function (e.g.,~\cite[Lemma 4]{mcmahan2017survey}).

Now we proceed to the more challenging task of bounding the expected regret.  We start by noting that the expected loss on any round $t$ is simply $\E[\iprod{c_t, \hat x_t}] = -p_t\langle c_t, h_t\rangle + (1-p_t)\langle c_t,x_t\rangle$. Therefore the expected regret for a fixed $u$ is:
\begin{align*}
    \sum_{t=1}^T \E[\langle c_t, \hat x_t - u\rangle]&=\sum_{t=1}^T p_t\langle c_t, -h_t-x_t\rangle +\langle c_t, x_t-u\rangle 
    \le \sum_{t=1}^T p_t(-\alpha \|c_t\|^2-\langle c_t, x_t\rangle) +\langle c_t, x_t-u\rangle,
\end{align*}
where we have used the fact that the hint $h_t$ is $\alpha$-good.  The main claim is then the following.

\begin{lemma}\label{lem:regret-main}
For the choice of $p_t, x_t$ defined in
$\less$ (Algorithm~\ref{alg:lesshints}), we have
\[
\sum_{t=1}^T p_t(-\alpha \|c_t\|^2-\langle c_t, x_t\rangle) +\langle c_t, x_t-u\rangle \le \frac{78 + 38\log(1+\|c\|^2_{1:T})}{\alpha}=O \left( \frac{1 + \log (1+\sigma_{1:T})}{\alpha} \right).
\]
\end{lemma}
Assuming this, the bound on expected regret easily follows, completing the proof.
\end{proof}
We now focus on proving Lemma~\ref{lem:regret-main}.
\begin{proof}[Proof of Lemma~\ref{lem:regret-main}]
The key idea is to prove the existence of a valid-in-hindsight sequence $(q_t)_{t=1}^T$ (Definition~\ref{def:valid-seq}) such that when $p_t$ on the LHS is replaced with $q_t$, the sum is $O(\log(T)/\alpha)$.  The guarantee of Algorithm~\ref{alg:ogd} (i.e., Theorem~\ref{thm:ogd}) then completes the proof. Specifically, since we set $\lambda = 10/\alpha$ and $z_t = (-\alpha \|c_t\|^2-\langle c_t, x_t\rangle)$, Theorem \ref{thm:ogd} guarantees that for any valid-in-hindsight sequence $(q_t)_{t=1}^T$, we have:
\begin{align}
    & \sum_{t=1}^T p_t(-\alpha \|c_t\|^2-\langle c_t, x_t\rangle) +\langle c_t, x_t-u\rangle \nonumber \\
    & \le \sum_{t=1}^T q_t(-\alpha \|c_t\|^2-\langle c_t, x_t\rangle) +\langle c_t, x_t-u\rangle +\frac{10}{\alpha}\left(1+3\log(1+\|c\|^2_{1:T})\right). \label{eq:replace-p-by-q}
\end{align}

Let us define $S$ to be the largest index in $[T]$ such that $\norm{c_{1:S}} \le \frac{\alpha}{4} (1+\sigma_{1:S})$. Firstly, by Theorem \ref{thm:ftrlregret} (part 2), for any such index $S$ we have: %
\begin{equation}
\sum_{t=1}^T \langle c_t, x_t- u\rangle \le  \frac{\sqrt{1+\sigma_{1:S}}}{2}+ \frac{18+8\log(1+\sigma_{1:T})}{\alpha} + \|c_{1:S}\|+\sum_{t=1}^S \langle c_t, x_t\rangle. \label{eq:excess-regret}
\end{equation}
Let $\Delta := \frac{\sqrt{1+\sigma_{1:S}}}{2} +  \|c_{1:S}\| + \sum_{t=1}^S \langle c_t, x_t\rangle$ denote the ``excess'' over the logarithmic term. Note that by Theorem~\ref{thm:ftrlregret} (part 1), for a vector $u = \frac{-c_{1:S}}{\|c_{1:S}\|}$, we have $\|c_{1:S}\| + \sum_{t=1}^S \langle c_t, x_t\rangle = \sum_{t=1}^S \iprod{c_t, x_t - u} \leq 4.5 \sqrt{1+\sigma_{1:S}}$. Thus, we have $\Delta \le 5 \sqrt{1+\sigma_{1:S}}$. 

Now, note that if $1 + \sigma_{1:S} \le \frac{100}{\alpha^2}$, we have $5 \sqrt{1+\sigma_{1:S}} \le 50/\alpha$. Thus, by setting $q_t=0$ for all $t$ (which is clearly valid-in-hindsight), the proof follows. Further, if $\Delta \le 1$, then clearly we can again set $q_t=0$ for all $t$ to complete the proof. Thus, we assume in the remainder of the proof that $1 + \sigma_{1:S} > \frac{100}{\alpha^2}$ and $\Delta >1$. 

Our goal now is to construct a valid-in-hindsight switch-once sequence that has value $q_t = \delta \in [0, 1]$ for $t \le S$ and $q_t = 0$ for $t > S$ such that we also have :
\begin{equation}  \delta \left( \sum_{t =1}^S \alpha \sigma_t + \iprod{c_t, x_t} \right) \ge 5 \sqrt{1 + \sigma_{1:S}}. \label{eq:delta-def}
\end{equation}

First, let us understand the term in the parentheses on the LHS above.  We bound this using the following claim.

\noindent {\em Claim. }  $\alpha \cdot \sigma_{1:S} + \sum_{t\le S} \iprod{c_t, x_t} \ge \frac{\alpha}{2}  (1 + \sigma_{1:S})$.  

\begin{proof}[Proof of claim]
Suppose that we have $\sum_{t \le S} \iprod{c_t, x_t} < \frac{\alpha}{2} ( 1- \sigma_{1:S})$.  By definition of $S$, we have
\begin{align*}
\Delta = \frac{\sqrt{1+\sigma_{1:S}}}{2} +  \|c_{1:S}\| + \sum_{t=1}^S \langle c_t, x_t\rangle &\le  \frac{\sqrt{1 + \sigma_{1:S}}}{2} + \frac{\alpha}{4} (1+\sigma_{1:S}) + \frac{\alpha}{2} (1-\sigma_{1:S}) \\
&= \frac{\sqrt{1 + \sigma_{1:S}}}{2} - \frac{\alpha}{4} (1 + \sigma_{1:S}) + \alpha. 
\end{align*}
From our assumption that $\sqrt{1+\sigma_{1:S}} \ge \frac{10}{\alpha}$, the RHS above is $\le \alpha$, which in turn is at most 1. Since we assumed $\Delta>1$, this is a contradiction so the claim holds.
\end{proof}

Thus in order to satisfy~\eqref{eq:delta-def}, we simply choose
\[ \delta = \frac{10}{\alpha \sqrt{1 + \sigma_{1:S}}}. \]
By assumption, this lies in $[0,1]$, and further, for $\lambda = 10/\alpha$, the $(q_t)$ defined above is a valid-in-hindsight sequence.   Combining the fact that $\Delta \le 5 \sqrt{1 + \sigma_{1:S}}$ with~\eqref{eq:delta-def}, we have that
\[ \sum_{t=1}^T q_t(-\alpha \|c_t\|^2-\langle c_t, x_t\rangle) +\langle c_t, x_t-u\rangle \le \frac{18 + 8\log (1+\sigma_{1:T})}{\alpha}.\]

Now appealing to the guarantee of Theorem~\ref{thm:ogd} as discussed earlier, we can replace $q_t$ with $p_t$ by suffering an additional logarithmic term on the RHS. Combining all these cases completes the proof of Lemma~\ref{lem:regret-main}, and thus also the proof of Theorem~\ref{thm:fullexpectedregret}.
\end{proof}

\section{Extensions and applications}
\label{sec:ext}
In the following subsections, we extend the analysis of Algorithm~\ref{alg:lesshints} to disparate settings: we consider robustness to uninformative or ``bad'' hints, the more classical ``optimistic'' regret setting, and online learning with abstention.
\subsection{Bad hints}

First, we extend Theorem \ref{thm:fullexpectedregret} to consider the case $\bad \neq \emptyset$  by carefully accounting for the regret incurred during rounds where $t \in \bad$ and making crucial use of the shrinking domain $D_t$.

\begin{restatable}{theorem}{thmlesshintswithbad}\label{thm:lesshintswithbad}
For any $\cB$, 
\begin{align*}
    \E[\regret_{\less, \alpha}(\vec{c})] &\le \frac{78+38\log(1+\|c\|^2_{1:T})}{\alpha}  + 40\sqrt{\sum_{t\in \bad }\|c_t\|^2} +\frac{20}{\alpha}\sqrt{\sum_{t\in \bad}\|h_t\|^2}\sqrt{\log(1+\|c\|^2_{1:T})}\\
    &= O\left(\frac{\sqrt{|\bad|} \log T}{\alpha}\right), \mbox{ and } \qquad \E[\query_{\less, \alpha}(\vec{c})] \le  20\sqrt{\|c\|^2_{1:T}}.
\end{align*}
\end{restatable}

\subsection{Optimistic bounds}

Next, we show that our results have implications for \emph{optimistic} regret bounds (e.g., \cite{hazan2010extracting, chiang2012online, rakhlin2013online, steinhardt2014eg, mohri2016accelerating}). The standard optimistic regret bound takes the form $O(\sqrt{\sum_{t=1}^T \|c_t-h_t\|^2})$. We will show that the same result can be obtained (up to logarithmic factors) even while only looking at $O(\sqrt{T})$ hints. The approach is very simple: if we set $\alpha =\frac{1}{4}$, then a little calculation shows that for $t\in \bad$, $\|c_t\|^2 + \|h_t\|^2=O(\|c_t-h_t\|^2)$, so that Theorem \ref{thm:lesshintswithbad} directly implies the desired result.

\begin{restatable}{theorem}{thmoptimistic}\label{thm:optimistic}
Set  $\alpha=\frac{1}{4}$.  Then
\begin{align*}
    \E[\regret_{\less, \alpha}(\vec{c})] &\le 312+152\log(1+\|c\|^2_{1:T}) +80\left(1+   \sqrt{\log(1+\|c\|^2_{1:T})}\right)\sqrt{\sum_{t\in \bad}\|c_t-h_t\|^2}\\
    &= O\left(\log(T) + \sqrt{\sum_{t=1}^T \|c_t-h_t\|^2\log(T)}\right),
    \mbox{ and } \quad
    \E[\query_{\less, \alpha}(\vec{c})] \le  20\sqrt{\|c\|^2_{1:T}}.
\end{align*}
\end{restatable}

\subsection{Online learning with abstention}

Finally, we apply our algorithm to the problem of online learning \emph{with abstentions}. In this variant of the OLO game, instead of being provided with hints, the learner is allowed to ``abstain'' on any given round. When the learner abstains, it receives a loss of $-\alpha \|c_t\|^2$ for some known $\alpha$ but pays a query cost of $\alpha \|c_t\|^2$. The regret is again the total loss suffered by the learner minus the total loss suffered by the best fixed adversary, which does not abstain. This setting is very similar to the scenario studied by \cite{neu2020fast} in the expert setting, but in addition to moving from the simplex to the unit ball, we ask for a more detailed guarantee from the learner: it is not allowed to abstain too often, as measured by the query cost. In this setting, our Algorithm \ref{alg:lesshints} works essentially out-of-the-box: every time the algorithm queries a hint, we instead simply choose to abstain. This procedure then guarantees:
\begin{restatable}{theorem}{thmabstaining}\label{thm:abstaining}
In the online learning with abstention model, the variant of Algorithm~\ref{alg:lesshints} that abstains whenever the original algorithm would ask for a hint guarantees expected regret at most:
\begin{align*}
    \frac{78 + 38\log(1+\|c\|^2_{1:T})}{\alpha}=O \left( \frac{1 + \log (1+\sigma_{1:T})}{\alpha} \right).
\end{align*}
Further, the expected query cost is at most $20\sqrt{\|c\|^2_{1:T}}$.
\end{restatable}
\begin{proof}
Since we abstain with probability $p_t$ and otherwise play $x_t$, the expected regret is $\sum_{t=1}^T -p_t \alpha \|c_t\|^2 + (1-p_t)\langle c_t, x_t\rangle - \langle c_t, u\rangle$. Thus the regret bound follows directly from Lemma \ref{lem:regret-main}. The query cost bound follows the identical argument as Theorem \ref{thm:fullexpectedregret}.
\end{proof}
\section{Lower bounds}
\label{sec:lb}

In this section we first show that the regret bound obtained in Theorem~\ref{thm:fullexpectedregret} is essentially tight.  Next, we show that randomness is necessary in our algorithms.  

\begin{restatable}{theorem}{thmlbone}
\label{thm:lb1}
Let $\alpha \in (0,1]$ be any parameter, and suppose $\cA$ is an algorithm for OLO with hints that makes $o\left( \frac{\sqrt{T}}{\alpha} \right)$ hint queries. Then there exists a sequence of cost vectors $c_t$ and hints $h_t$ of unit length, such that (a) in any round $t$ where a hint is asked, $\iprod{c_t, h_t} \ge \alpha \norm{c_t}^2$, and (b) the regret of $\cA$ on this input sequence is $\Omega(\sqrt{T})$.
\end{restatable}
\begin{proof}
We will construct a distribution over inputs $\{c_t, h_t\}$ and argue that any deterministic algorithm has an expected regret $\Omega(\sqrt{T})$ for inputs drawn from this distribution. By the minmax theorem, we then have a lower bound for any (possibly randomized) algorithm $\cA$.

We consider two-dimensional inputs. At each step, $h_t$ is chosen to be a uniformly random vector on the unit circle (in $\R^2$). The cost $c_t$ is set to be $\alpha h_t \pm \sqrt{1-\alpha^2} ~h_t^\perp$, where $h_t^\perp$ is a unit vector orthogonal to $h_t$, and the sign is chosen uniformly at random. Now for any deterministic algorithm, if $\cA$ queries $h_t$ at time $t$, then it can play a point $a h_t + b h_t^\perp $, for scalars $a, b$. In expectation, this has inner product $a \alpha$ with $c_t$, and thus the expected cost incurred by $\cA$ in this step is $\ge -\alpha$ (since $|a| \le 1$). If $\cA$ does not query $h_t$, then $c_t$ is simply a random unit vector, and thus the expected cost incurred by $\cA$ in this step is 0.

Next, consider the value $\min_{\norm{u} \le 1} \sum_t \iprod{c_t, u}$, i.e., the best cost in hindsight;  this is clearly $-\norm{\sum_t c_t}$. By construction, $c_t$ is a uniformly random vector on the unit circle in $\R^2$ (and the choices are independent for different $t$). Thus we have $\E[ \norm{\sum_t c_t}] \ge \Omega(\sqrt{T})$. (This follows from properties of sums of independent random unit vectors; see Supplementary Material for a proof.)

Thus, if the algorithm makes $K$ queries, then the expected regret is at least $-K\alpha + \Omega(\sqrt{T})$. This quantity is $\Omega(\sqrt{T})$ as long as $K = o\left( \frac{\sqrt{T}}{\alpha} \right)$, thus completing the proof.
\end{proof}

We next show that for {\em deterministic} algorithms, making $O(\sqrt{T})$ hint queries is insufficient for obtaining $o(\sqrt{T})$ regret, even if the hints provided are always $1$-good (i.e., hints are perfect). 
\begin{restatable}{theorem}{thmlowerbounddet}\label{thm:lowerbound_det}
Let $\cA$ be any deterministic algorithm for OLO with hints that makes at most $ C \sqrt{T} < T/2$ queries, for some parameter $C > 0$. Then there is a sequence  cost vectors $c_t$ and hints $h_t$ of unit length such that (a) $h_t = c_t$ whenever $\cA$ makes a hint query, and (b) the regret of $\cA$ on this input sequence is at least $\frac{\sqrt{T}}{2(1+C)}$.  
\end{restatable}
We remark that by setting $C$ appropriately, we can also show that for a deterministic algorithm to achieve logarithmic regret, it needs to make $\Omega \left( \frac{T}{\log T} \right)$ queries.

\section{Unconstrained setting}
\label{sec:unconstrained}

\begin{algorithm}[H]
   \caption{Algorithm with hints (unconstrained case).}
   \label{alg:unconstrained}
   \begin{algorithmic}
   \REQUIRE Parameters $\epsilon$, $\alpha$, $K$, $d$-dimensional unconstrained OLO algorithm $\unconstrained$, one-dimensional unconstrained OLO algorithm $\unconstrainedoned$ guaranteeing (\ref{eqn:unconstrainedstandardbound})
    \FOR{$t=1,\dots,T$}
        \STATE \COMMENT{Randomized version}
        $\mathbbm{1}_{t} \gets 1$ with probability $\min\left(1,\frac{K}{\alpha\sqrt{1+\|c\|^2_{1:t-1}}}\right)$; 0 otherwise.
        \STATE \COMMENT{Deterministic version}
        $\mathbbm{1}_{t} \gets 1$ iff $1+\sum_{\tau=1}^{t-1}i_\tau \langle c_\tau, h_\tau\rangle \le K\sqrt{1+\|c\|^2_{1:t-1}}$.
        \STATE Receive $w_t\in \R^d$ from $\unconstrained$; Receive $y_t\in \R$ from $\unconstrainedoned$.
        \STATE If $\mathbbm{1}_{t} = 1$, get hint $h_t$
        \STATE Play $x_t = w_t - \mathbbm{1}_{t} h_t y_t$; Receive cost $c_t$.
        \STATE Send $c_t$ to $\unconstrained$ as $t$th cost; Send $g_t=-\mathbbm{1}_{t}\langle h_t, c_t\rangle\in \R$ to $\unconstrainedoned$ as $t$th cost.
	\ENDFOR
\end{algorithmic}
\end{algorithm}

In this section, we consider \emph{unconstrained} online learning in which the domain is all of $\R^d$. In this setting, it is unreasonable to define the regret using a supremum over all comparison points $u\in  \R^d$ as this will invariably lead to infinite regret. Instead, we bound the regret as a function of $u$. For example, when hints are \emph{not} available, standard results provide bounds of the form \cite{cutkosky2018black,kempka2019adaptive, van2019user,mhammedi2020lipschitz, chen2021impossible}:
\begin{align}
    \sum_{t=1}^T \langle c_t, x_t-u\rangle &\le \epsilon+A\|u\|\sqrt{\sum_{t=1}^T \|c_t\|^2\log(\|u\|T/\epsilon+1)} + B\|u\|\log(\|u\|T/\epsilon+1),
    \label{eqn:unconstrainedstandardbound}
\end{align}
for constants $A$ and $B$ and any user-specified $\epsilon$. Using such algorithms as building blocks, we design an algorithm with $O(\sqrt{T})$ expected query cost and for all comparators $u$, regret is $\tilde O(\|u\|/\alpha)$.

The algorithm is somewhat simpler than in the constrained case: we take an ordinary algorithm that does not use hints and subtract the hints from its predictions. Intuitively, each subtraction decreases the regret by $\alpha \|c\|^2$, so we need only $O(\sqrt{T})$ such events. With constraints, this is untenable because subtracting the hint might violate the constraint, but there is no problem in the unconstrained setting. Instead, the primary difficulty is that we need the regret to decrease not by $\alpha \|c\|^2$ but by $\alpha\|u\| \|c\|^2$ for some \emph{unknown} $\|u\|$. This is accomplished by learning a scaling factor $y_t$ that is applied to the hints.

Moreover, in the case that all hints are \emph{guaranteed} to be $\alpha$-good, we devise a \emph{deterministic} algorithm for the unconstrained setting. In light of Theorem \ref{thm:lowerbound_det}, this establishes a surprising separation between the unconstrained and constrained settings. For the deterministic approach, we directly measure the total query cost and simply query a hint whenever the cost is less than our desired budget. Note that this strategy fails if we allow bad hints as the adversary could provide a bad hint every time we ask for a hint. The full algorithm is presented in Algorithm \ref{alg:unconstrained}, with the randomized and deterministic analyses provided by Theorems \ref{thm:unconstrained} and \ref{thm:unconstraineddeterministic}. 

\begin{restatable}{theorem}{thmunconstrained}\label{thm:unconstrained}
% Suppose that $\unconstrained$ is an unconstrained online linear optimization algorithm that outputs $w_t\in \R^d$ in response to costs $c_1,\dots,c_{t-1}\in \R^d$ satisfying $\|c_\tau\|\le 1$ for all $\tau$ and guarantees for some constants $A$ and $B$ for all $u\in \R^d$:
% \begin{align*}
%     \sum_{t=1}^T \langle c_t, x_t -u\rangle &\le \epsilon+A\|u\|\sqrt{\sum_{t=1}^T \|c_t\|^2\log(\|u\|T/\epsilon+1)} + B\|u\|\log(\|u\|T/\epsilon+1)
% \end{align*}
% where $\epsilon$ is an arbitrary user-specified constant.
% Further, suppose $\unconstrainedoned$ is an unconstrained online linear optimization algorithm that outputs $y_t\in \R$ in response to $g_1,\dots,g_{t-1}\in \R$ satisfying $|g_\tau|\le 1$ for all $\tau$ and guarantees for all $y_\star\in \R$:
% \begin{align*}
%     \sum_{t=1}^T g_t(y_t-y_\star)&\le \epsilon+A|y_\star|\sqrt{\sum_{t=1}^T g_t^2\log(|y_\star|T/\epsilon+1)} + B|y_\star|\log(|y_\star|T/\epsilon+1)
% \end{align*}
% Then 
The randomized version of Algorithm \ref{alg:unconstrained} guarantees an expected regret at most:
\begin{align*}
2\epsilon+\tilde O\left(\frac{\|u\|\sqrt{\log(\|u\|T/\epsilon)}\left[K+\frac{\log(\|u\|T/\epsilon)\log\log(T\|u\|/\epsilon)}{K}+\sqrt{\sum_{t\in \bad}\|h_t\|^2\log(T)}\right]}{\alpha}\right),
\end{align*}
with expected query cost at most $2K\sqrt{\|c\|^2_{1:T}}$.
\end{restatable}

\begin{restatable}{theorem}{thmunconstraineddeterministic}\label{thm:unconstraineddeterministic}
% Under the same assumptions as Theorem \ref{thm:unconstrained}, and subject to the additional assumption 
If $\bad=\emptyset$, then the deterministic version of Algorithm \ref{alg:unconstrained} guarantees:
\begin{align*}
    \sum_{t=1}^T \langle c_t, x_t - u\rangle&\le2\epsilon + O\left(\frac{\|u\|\sqrt{\log(\|u\|T/\epsilon+1)}}{\alpha} + \frac{\|u\|\log^{3/2}(\|u\|T/\epsilon)\log\log(\|u\|T/\epsilon)}{K}\right),
\end{align*}
with a query cost at most $2K\sqrt{\|c\|^2_{1:T}}$.
\end{restatable}

\section{Conclusions}
\label{sec:conc}

In this paper, we consider OLO where an algorithm has query access to hints, in both the constrained and unconstrained settings.  Surprisingly, we show that it is possible to obtain logarithmic expected regret by querying for hints at only $O(\sqrt{T})$ time steps. Our work also demonstrates an intriguing separation between randomized and deterministic algorithms for constrained online learning.  While our algorithms need to know $\alpha$, an  open question is to obtain an algorithm that can operate without knowing $\alpha$.  Extending our model to the bandit setting is also an interesting research direction. 

\begin{ack}
Aditya Bhaskara acknowledges the support from NSF awards 2047288 and 2008688. 
%Do {\bf not} include this section in the anonymized submission, only in the final paper. You can use the \texttt{ack} environment provided in the style file to autmoatically hide this section in the anonymized submission.
\end{ack}

\bibliographystyle{plain}
\bibliography{shints}

\newpage
\appendix
\section*{Supplementary Material}

\section{Missing proofs}

\begin{proposition}
\label{prop:loginequality}
For any arbitrary non-negative real numbers $a_1, \ldots, a_T$ , we have
\[\sum_{t=1}^T \frac{a_t}{1 + a_{1:t}} \leq \log(1 + a_{1:T}).\]
\end{proposition}

\begin{proof}
For any $a, b > 0$, we have
\begin{equation}
\frac{a}{b+a} = \int_{x=0}^a \frac{1}{b+a} dx \leq \int_{x=0}^a \frac{1}{b+x} dx = \log(b+a) - \log(b).
\label{eq:logineq}
\end{equation}

The proof now follows from induction. The base case of $t=1$ follows directly from \eqref{eq:logineq} with $a$ set to $a_1$ and $b$ set to 1. Assuming that the inequality holds for $T-1$, let us consider the induction step. \begin{align*}
\sum_{t=1}^T \frac{a_t}{1 + a_{1:t}} &= \frac{a_T}{1+a_{1:T}} + \sum_{t=1}^{T-1} \frac{a_t}{1 + a_{1:t}} 
\leq \frac{a_T}{1+a_{1:T}} + \log(1 + a_{1:T-1})
\leq \log(1 + a_{1:T}),
\end{align*}
where the last inequality again follows from \eqref{eq:logineq} with $a$ set to $a_T$ and $b$ set to $1 + a_{1:T-1}$.
\end{proof}

\begin{proposition}
\label{prop:minimizer-closed-form}
Consider any $c \in \R^d$ and $r \geq 0$ and let $y = \argmin_{\|x\| \leq 1} \frac{r}{2}\|x\|^2 + \iprod{c, x}$. Then, if $\|c\| \geq r$, we have $y = \frac{-c}{\|c\|}$.
\end{proposition}

\begin{proof}
Consider $f(x) = \frac{r}{2}\|x\|^2 + \iprod{c, x}$. For any $\|x\| \leq 1$, we have the following.
\begin{align*}
    f(x) &\geq \frac{r}{2}\|x\|^2 - \|c\|\|x\| 
    \geq \min_{\|z\| \leq 1} \left(\frac{r}{2}\|z\|^2 - \|c\|\|z\|\right),
\intertext{since $\|c\| \geq r$, it is an easy exercise to verify that the RHS is minimized at $\|z\| = 1$ and thus}
f(x) &\geq \frac{r}{2} - \|c\|.
\end{align*}
On the other hand, substituting $y = \frac{-c}{\|c\|}$, we have $f(y) = \frac{r}{2} - \|c\|$ and the proposition follows.
\end{proof}

\begin{lemma}\label{lem:app:paleyzig}
Let $c_1, \dots, c_n$ be independent random unit vectors in $\R^d$ (distributed uniformly on the sphere), for some parameters $n, d$, and let $Z = \sum_{t=1}^n c_t$  Then we have $\E[ \norm{Z}] \ge \Omega(\sqrt{n})$.
\end{lemma}
\begin{proof}
First, we note that since $c_t$ are independent, we have
\[\E[ \norm{Z}^2] = \sum_{t=1}^n \norm{c_t}^2 = n. \]

We also have
\[ \E[ (\norm{Z}^2)^2 ] = \E \big[ \big( \sum_i \norm{c_i}^2 + \sum_{i \ne j} \iprod{c_i, c_j} \big)^2 \big] \le n^2 + \sum_{i \ne j} \E[ \iprod{c_i, c_j}^2 ] \le 2n^2.   \]

Thus by applying the Paley--Zygmund inequality to the random variable $\norm{Z}^2$, we have  $\Pr[\norm{Z}^2 \geq n/4] = \Omega(1)$, and thus $\Pr[ \norm{Z} \ge \sqrt{n}/2] = \Omega(1)$. Thus the expected value is $\Omega(\sqrt{n})$.
\end{proof}

\section{A sharper analysis of FTRL}\label{app:ftrl}
Our goal in this section is to prove Theorem~\ref{thm:ftrlregret}. As a first step, let us define $\psi_t (x) = \iprod{c_t, x} + \frac{r_t}{2} \norm{x}^2$, (with the understanding that $c_0 = 0$) so that by definition, we have \[ x_{t+1} = \argmin_{ \norm{x} \le 1} \psi_{0:t} (x). \]

\begin{lemma}\label{lem:ftl-lemma}
Let $\psi_t, x_t$ be as defined above. Then for any $m \in [T]$ and any vector $u$ with $\norm{u} \le 1$, we have
\[ \psi_{0:m} (x_{m+1}) + \sum_{t = m+1}^T \psi_t (x_{t+1}) \le \psi_{0:T} (u). \]
\end{lemma}
When $m=0$, the lemma is usually referred to as the FTL lemma (see e.g., \cite{kalai2005efficient}), and is proved by induction. Our proof follows along the same lines.
\begin{proof}
From the definition of $x_{T+1}$ (as the minimizer), we have
\[ \psi_{0:T}(u) \ge \psi_{0:T} (x_{T+1}). \]
Now, we can clearly write $\psi_{0:T}(x_{T+1}) = \psi_{T} (x_{T+1}) + \psi_{0:T-1} (x_{T+1})$.  Next, observe that from the definition of $x_T$, we have $\psi_{0:T-1} (x_{T+1}) \ge \psi_{0:T-1} (x_{T})$.  Plugging this above,
\[  \psi_{0:T}(u) \ge \psi_{T} (x_{T+1}) + \psi_{0:T-1} (x_{T}). \]
Once again, writing $\psi_{0:T-1}( x_{T} ) = \psi_{T-1} (x_T) + \psi_{0:T-2} (x_{T})$ and now using the definition of $x_{T-1}$, we obtain
\[ \psi_{0:T}(u) \ge \psi_{T} (x_{T+1}) + \psi_{T-1} (x_{T}) + \psi_{0:T-2} (x_{T-1}). \]
Using the same reasoning again, and continuing until we reach the subscript $0$:$m$ in the last term of the RHS, we obtain the desired inequality.
\end{proof}

We are now ready to prove Theorem~\ref{thm:ftrlregret}.

\begin{proof}[Proof of Theorem~\ref{thm:ftrlregret}]
Let us focus on Part 2 for now (see Lemma~\ref{lem:ftrl-basic-regret} for Part 1).  Note that we can rearrange the bound we wish to prove, i.e.,~\eqref{eq:ftrl-to-show}, as follows.  Let $z$ be the unit vector in the direction of $-c_{1:S}$, so that $ -\norm{c_{1:S}} = \sum_{t = 1}^S \iprod{c_t, z}$.  Then~\eqref{eq:ftrl-to-show} can be rewritten as
\begin{equation*}
\sum_{t=1}^S \iprod{c_t, z - u} + \sum_{t > S} \iprod{c_t, x_t - u} \le \frac{ \sqrt{1 + \sigma_{1:S}}}{2} + \frac{18 + 8 \log(1 + \sigma_{1:T})}{\alpha}.
\end{equation*}

As a first step, we observe that $\iprod{c_{1:S}, z} \le \iprod{c_{1:S}, x_{S+1}}$; indeed, $\norm{x_{S+1}} \le 1$ by definition.  Thus, it will suffice to prove that
\begin{equation}
\sum_{t=1}^S \iprod{c_t, x_{S+1} - u} + \sum_{t > S} \iprod{c_t, x_t - u} \le \frac{\sqrt{1 + \sigma_{1:S}}}{2} + \frac{18 + 8 \log(1 + \sigma_{1:T})}{\alpha}. \label{eq:ftrl-toshow2}
\end{equation}
For proving this, we first appeal to Lemma~\ref{lem:ftl-lemma}.  Instantiating the lemma with $m=S$ and plugging in the definition of $\psi$, we get
\[ \iprod{c_{0:S}, x_{S+1}} + \frac{r_{0:S}}{2} \norm{x_{S+1}}^2 + \sum_{t > S} \iprod{c_t, x_{t+1}} + \frac{r_t}{2} \norm{x_{t+1}}^2 \le \iprod{c_{0:T}, u} + \frac{r_{0:T}}{2} \norm{u}^2. \]
Noting that $c_0 = 0$ and rearranging, we get:
\begin{align*}
\sum_{t=1}^S \iprod{c_t, & x_{S+1} - u} + \sum_{t > S} \iprod{c_t, x_t - u} \\
&\le \frac{r_{0:S}}{2} (\norm{u}^2 - \norm{x_{S+1}}^2) + \sum_{t > S} \left( \frac{r_t}{2} (\norm{u}^2 - \norm{x_{t+1}}^2) + \iprod{c_t, x_t - x_{t+1}} \right) \\
&\le \frac{r_{0:S}}{2} +  \sum_{t > S} \left( \frac{r_t}{2} (\norm{u}^2 - \norm{x_{t+1}}^2) + \iprod{c_t, x_t - x_{t+1}} \right).
\end{align*}
The LHS matches the quantity we wish to bound in~\eqref{eq:ftrl-toshow2}, and thus let us analyze the RHS quantity, which we denote by $\mathscr{Q}$. 

The next observation is that if $t >S$ and $\sqrt{1 + \sigma_{1:t}} \ge \frac{4}{\alpha}$, then the vector $x_{t+1}$ has norm exactly $1$.  This can be shown as follows. If $t > S$, by the definition of $S$, we have $\norm{c_{1:t}} > \frac{\alpha}{4} (1+\sigma_{1:t})$.  Thus, the vector $- c_{1:t} / \sqrt{1+\sigma_{1:t}}$ has norm $\ge 1$.  From the definition of $x_{t+1}$ (see \eqref{eq:ftrl-iterate}), this means that the global minimizer (without the constraint $\norm{x} \le 1$) of the quadratic form is a point outside the ball, and thus the minimizer of the constrained problem is its projection, which is thus a unit vector. See Proposition~\ref{prop:minimizer-closed-form} for further details.  We next have the following claim.

{\em Claim.}  Let $M$ be the smallest index $>S$ for which $\sqrt{1+\sigma_{1:M}} \ge \frac{4}{\alpha}$.  Then 
\[ \sqrt{1 + \sigma_{1:M-1}} \le \max \left\{ \sqrt{1 + \sigma_{1:S}},  \frac{4}{\alpha} \right\}. \]
The claim follows by a simple case analysis. If $M = S+1$, then clearly the LHS is $\sqrt{1+\sigma_{1:S}}$.  Otherwise, from the definition of $M$, we have the desired bound.

Let us get back to bounding the quantity $\mathscr{Q}$ defined above.  We split the sum into indices $\le M-1$ and $\ge M$.  The nice consequence of the observation above is that for all $t \ge M$, as $\norm{x_{t+1}} =1$, we have $\norm{u}^2 - \norm{x_{t+1}}^2 \le 0$, thus the term disappears.  Also, for $t < M$, we use the simple bound $\frac{r_t}{2} (\norm{u}^2 - \norm{x_{t+1}}^2) \le \frac{r_t}{2}$.  This gives
\[ \mathscr{Q} \le \frac{r_{0:M-1}}{2} + \sum_{t=S+1}^T \iprod{c_t, x_t - x_{t+1}}. \]
Thus we only need to analyze the summation on the RHS.  To bound the summation $\sum_{t=S+1}^T \iprod{c_t, x_t - x_{t+1}}$ consider two cases for $M$ separately: either $M=S+1$ or $M>S+1$. If $M=S+1$, then by Proposition~\ref{prop:stabilityfromlength},   $\sum_{t=S+1}^T \iprod{c_t, x_t - x_{t+1}}\le \frac{8}{\alpha} \log (1 + \sigma_{1:T})$. Alternatively, if $M>S+1$, let us break the summation into terms with $t \le M-1$ and terms with $t \ge M$.  Proposition~\ref{prop:stabilityfromr} lets us bound the sum of the terms corresponding to $t \le M-1$ by $4 \sqrt{\sigma_{1:M-1}} < 4 r_{0:M-1}\le \frac{16}{\alpha}$, where the last step is by definition of $M$ and using the fact that $M-1>S$. Then Proposition~\ref{prop:stabilityfromlength} lets us bound the sum of the terms with $t \ge M$ by $\frac{8}{\alpha} \log (1 + \sigma_{1:T})$. Thus in all cases we have:
\begin{align*}
    \mathscr{Q} \le  \frac{r_{0:M-1}}{2} + \frac{16}{\alpha} + \frac{8}{\alpha}\log(1+\sigma_{1:T})\le \frac{\sqrt{1+\sigma_{1:S}}}{2} + \frac{18}{\alpha} + \frac{8}{\alpha}\log(1+\sigma_{1:T}),
\end{align*}
where in the last step we used the claim and bounded the maximum with a sum.
\end{proof}

\subsection{Auxiliary lemmas}\label{sec:app:ftrlaux}
\begin{proposition}\label{prop:stabilityfromr}
For any time step $t \leq T$, the iterates of the FTRL procedure satisfy:
\begin{align*}
    \|x_t-x_{t+1}\|\le \frac{2\|c_t\|}{\sqrt{1+\sigma_{1:t-1}}}.
\end{align*}
Furthermore, in any time interval $[A, B]$ with $1 \le A \le B \le T$, we have
\[  \sum_{t=A}^B \iprod{c_t, x_t - x_{t+1}} \le 4 \left( \sqrt{\sigma_{1:B}} - \sqrt{\sigma_{1:A-1}} \right). \]
\end{proposition}
\begin{proof}
Let us first show the first part.  Define $\psi_t(x) = \langle c_t, x\rangle + \frac{r_t}{2}\|x\|^2$. We will invoke  \cite[Lemma 7]{mcmahan2017survey}, using $\phi_1=\psi_{0:t-1}$ and $\phi_2=\psi_{0:t}$. We have that $\phi_1$ is 1-strongly convex with respect to the norm given by $\|x\|_{t-1}^2 = r_{0:t-1}\|x\|^2$ and $\psi_t=\phi_2-\phi_1$ is convex and $2\|c_t\|$ Lipschitz. Then, since $x_t=\argmin \phi_1$ and $x_{t+1}=\argmin\phi_2$, ~\cite[Lemma 7]{mcmahan2017survey} implies:
\begin{align*}
    \|x_t-x_{t+1}\|\le \frac{2\|c_t\|}{r_{0:t-1}} = \frac{2 \norm{c_t}}{\sqrt{1 + \sigma_{1:t-1}}}.
\end{align*}

We can then use this to show the ``furthermore'' part as follows.  For any $t$ in the range, we have
\[ \iprod{c_t, x_t - x_{t+1}} \le \norm{c_t} \norm{x_t - x_{t+1}} \le \frac{2 \sigma_t}{\sqrt{1+\sigma_{1:t-1}}} \le \frac{2\sigma_t}{\sqrt{\sigma_{1:t}}} \le 2 \int_{\sigma_{1:t-1}}^{\sigma_{1:t}} \frac{dy}{\sqrt{y}}, \]
where in the third inequality, we used the fact that $\sigma_t \le 1$, and in the last inequality, we upper bounded the term via an integral over an interval of length $\sigma_t$.  Summing this over $t$ in the interval $[A, B]$ thus gives 
\[ \sum_{t=A}^B \iprod{c_t, x_t -x_{t+1}} \le 2 \int_{\sigma_{1:A-1}}^{\sigma_{1:B}} \frac{dy}{\sqrt{y}} = 4 \left( \sqrt{\sigma_{1:B}} - \sqrt{\sigma_{1:A-1}} \right). 
\qedhere
\]
\end{proof}

\begin{proposition}\label{prop:stabilityfromlength}
Let $S$ be an index such that for all $t > S$, $\norm{c_{1:t}} \ge \frac{\alpha}{4}(1+c_{1:t})$, and let $t > S$ be an index for which the iterates $x_t$ and $x_{t+1}$ of the FTRL procedure are both unit vectors.  Then, 
\begin{align*}
    \|x_t-x_{t+1}\|\le \frac{8\|c_t\|}{\alpha( 1+\sigma_{1:t})}.
\end{align*}
Furthermore, let $M > S$ be an index such that $\norm{x_t} =1$ for all $t \ge M$.  Then,
\[  \sum_{t=M}^T \iprod{c_t, x_t - x_{t+1}} \le \frac{8}{\alpha} \log (1 + \sigma_{1:T}). \]
\end{proposition}
\begin{proof}
For simplicity, let us denote $g_t = c_{1:t-1}$ and $g_{t+1} = c_{1:t}$.  If the iterates of FTRL are unit vectors, we have
\[ x_t = -\frac{g_t}{\norm{g_t}} ~;~ x_{t+1} = - \frac{g_{t+1}}{\norm{g_{t+1}}}. \]
Thus their difference can be bounded as
\[ x_{t+1} - x_t = \left( \frac{g_t}{\norm{g_t}} - \frac{g_t}{\norm{g_{t+1}}} \right) + \left( \frac{g_t}{\norm{g_{t+1}}} - \frac{g_{t+1}}{\norm{g_{t+1}}} \right). \]
The second term clearly has norm $\le \frac{\norm{c_t}}{\norm{g_{t+1}}}$.  Let us bound the first term:
\[ \norm{g_t} \left| \frac{1}{\norm{g_t}} -\frac{1}{\norm{g_{t+1}}} \right| = \frac{|\norm{g_{t+1}} - \norm{g_t}|}{\norm{g_{t+1}}} \le \frac{\norm{c_t}}{\norm{g_{t+1}}}. \]
Note that in the last step, we used the triangle inequality.  Combining the two, we get
\[ \norm{x_{t+1} - x_{t}} \le \frac{2\norm{c_t}}{\norm{c_{1:t}}} \le \frac{8 \norm{c_t}}{\alpha (1+\sigma_{1:t})},\]
as desired.  Let us now show the ``furthermore'' part.  From our assumptions about $M$, we can appeal to the first part of the proposition, and as before, we have for any $t \ge M$,
\[ \iprod{c_t, x_t - x_{t+1}} \le \norm{c_t} \norm{x_t - x_{t+1}} \le \frac{8 \sigma_t}{\alpha(1 + \sigma_{1:t})} \le \frac{8}{\alpha} \int_{1 + \sigma_{1:t-1}}^{1+\sigma_{1:t}} \frac{dy}{y}. \]

Now, summing this inequality over $t \in [M, T]$ gives us
\[ \sum_{t=M}^T \iprod{c_t, x_t - x_{t+1}} \le \frac{8}{\alpha} \int_{1+\sigma_{1:M-1}}^{1+\sigma_{1:M}} \frac{dy}{y} \le \frac{8}{\alpha} \log (1 + \sigma_{1:T}).
\qedhere
\]
\end{proof}

The next lemma is a consequence of the standard FTRL analysis. We include its proof for completeness. This is also Part (1) of Theorem~\ref{thm:ftrlregret}.
\begin{lemma}\label{lem:ftrl-basic-regret}
For the FTRL algorithm described earlier,  for all $N \in [T]$ and for any vector $u$ with $\norm{u} \le 1$, we have 
\[ \sum_{t =1}^N \iprod{c_t, x_t - u} \le 4.5 \sqrt{1 + \sigma_{1:N}}.  \]
\end{lemma}
\begin{proof}
Suppose we use Lemma~\ref{lem:ftl-lemma} with $m=0$ and $T=N$, then we get:
\[ \sum_{t=0}^N \psi_t (x_{t+1}) \le \psi_{0:N} (u). \]
Plugging in the value of $\psi_t$,
\[ \sum_{t=1}^N \iprod{c_t, x_t - u} \le \sum_{t=0}^N \frac{r_t}{2} (\norm{u}^2 - \norm{x_{t+1}}^2) + \sum_{t=1}^N \iprod{c_t, x_t - x_{t+1}}. \]
Now, we use the naive bound of $r_{0:N}$ for the first summation on the RHS, and use Proposition~\ref{prop:stabilityfromr}  to bound the second summation by $r_{0:N}$.  This completes the proof.
\end{proof}

\section{Switch-once dynamic regret}
\thmogd*
\begin{proof}
The proof is analogous to that of OGD (e.g.,~\cite{zinkevich2003online}), but we need fresh ideas specific to our setup.  First, observe that since $q$ is a valid-in-hindsight sequence, we have $q_t \in D_t$ for all $t$.

Thus, we have
\begin{align}
(p_{t+1} - q_t)^2 &\le (p_t - \eta_t z_t - q_t)^2 \qquad \text{ (since projection only shrinks distances)} \notag \\
&= (p_t - q_t)^2 - 2\eta_t z_t (p_t - q_t) + \eta_t^2 z_t^2. \notag \\
\implies z_t (p_t - q_t) &\le \frac{(p_t - q_t)^2 - (p_{t+1} - q_t)^2}{2\eta_t} + \frac{\eta_t}{2} z_t^2. \label{eq:ogd-regret1}
\end{align}
We now need to sum (\ref{eq:ogd-regret1}) over $t$.  Note that the second term is easier to bound:
\begin{equation} \sum_{t=1}^T \frac{\eta_t}{2} z_t^2 \le \frac{\lambda}{2}  \sum_{t=1}^T  \frac{ 4\sigma_t}{1+ \sigma_{1:t}} \le2\lambda ~ \log (1+\sigma_{1:T}),\label{eq:ogd-regret2}
\end{equation}
where the last inequality uses Proposition \ref{prop:loginequality}. 
Suppose $S$ is the time step at which the switch occurs in the sequence $q$, and let $\delta$ be $q_1$ (i.e., the value in the non-zero segment). We split the first term as:
\begin{align}
\sum_{t=1}^T \frac{(p_t - q_t)^2 - (p_{t+1} - q_t)^2}{2\eta_t} &=  \sum_{t \le S} \frac{(p_t - \delta)^2 - (p_{t+1} - \delta)^2}{2\eta_t} + \sum_{t > S}  \frac{p_t^2 - p_{t+1}^2}{2\eta_t}. \label{eq:ogd-regret}
\end{align}
Next, by setting $\eta_0=\lambda$, writing 
\[  \frac{(p_t - \delta)^2 - (p_{t+1} - \delta)^2}{2\eta_t} = \frac{(p_t - \delta)^2}{2 \eta_{t-1}} - \frac{(p_{t+1} - \delta)^2}{2 \eta_{t}} + \frac{(p_{t}-\delta)^2}{2} \left( \frac{1}{\eta_{t}} - \frac{1}{\eta_{t-1}} \right), \]
and noting that $ \frac{1}{\eta_{t}} - \frac{1}{\eta_{t-1}} = \frac{\sigma_{t}}{\lambda}$, we can make the summation telescope. Doing a similar manipulation for the sum over $t > S$, the RHS of~\eqref{eq:ogd-regret} simplifies to:
\begin{align}
& \frac{(p_1 - \delta)^2}{2\eta_0} - \frac{(p_{S+1} - \delta)^2}{2\eta_{S}} + \frac{p_{S+1}^2}{2\eta_{S}} - \frac{p_{T+1}^2}{2\eta_{T}} + \sum_{t\le S} \frac{(p_{t}-\delta)^2 \sigma_{t}}{2\lambda} + \sum_{t> S} \frac{p_{t}^2 \sigma_{t}}{2\lambda} \notag \\
&\le \frac{1}{2\eta_0} +  \frac{|D_{S}|^2}{2\eta_{S}} + \sum_{t=1}^T \frac{|D_t|^2 \sigma_{t}}{2\lambda},\label{eq:telescope}
\end{align}
where $|D_t|$ is the length/diameter of the domain at time $t$, i.e., $|D_t|^2 = \min(1, \frac{\lambda^2}{1+\sigma_{1:t}})$.  The inequality holds because for all $t$, both $p_t$ and $q_t$ are in $D_t$.  Plugging in the values of $|D_t|$ and $\eta_t$, the first two terms in~\eqref{eq:telescope} are  at most $\lambda/2$ (because $\lambda\ge 1$). Thus plugging this back into~\eqref{eq:ogd-regret}, we get
\[  \sum_{t=1}^T \frac{(p_t - q_t)^2 - (p_{t+1} - q_t)^2}{2\eta_t} \le \lambda \left( 1 + \sum_{t=1}^T \frac{\sigma_{t}}{2 (1+\sigma_{1:t})} \right) .\]

Finally, using Proposition \ref{prop:loginequality}, the RHS above can be upper bounded by $\lambda \left( 1 + \frac{1}{2} \log (1+ \sigma_{1:T}) \right)$. 

Plugging this back into~\eqref{eq:ogd-regret1}, summing over $t$, and using~\eqref{eq:ogd-regret2}, we get
\[ \sum_t z_t (p_t - q_t)  \le \lambda \left( 1+ 3 \log (1 + \sigma_{1:T}) \right).
\qquad
\qedhere
\]
\end{proof}

\section{Proofs for Section \ref{sec:ext}}

\thmlesshintswithbad*
\begin{proof}
In the proof of Theorem \ref{thm:fullexpectedregret}, we exploited the fact that Lemma \ref{lem:regret-main} actually bounds the expected regret when $\bad=\emptyset$. However, when $\bad\ne\emptyset$, we have a more complicated relationship:
\begin{align*}
    \sum_{t=1}^T \E[\langle c_t, \hat x_t - u\rangle]&=\sum_{t=1}^T p_t\langle c_t, - h_t-x_t\rangle  + \langle c_t, x_t-u\rangle\\
    &\le \sum_{t\notin \bad}  p_t(-\alpha \|c_t\|^2 -\langle c_t,x_t\rangle) + \langle c_t, x_t-u\rangle + \sum_{t\in \bad}p_t\langle c_t, - h_t-x_t\rangle  + \langle c_t, x_t-u\rangle\\
    &=\sum_{t=1}^T p_t(-\alpha \|c_t\|^2 -\langle c_t,x_t\rangle) +\langle c_t, x_t-u\rangle+ \sum_{t\in \bad}-p_t(\langle c_t,h_t\rangle-\alpha\|c_t\|^2)\\
    &\le \sum_{t=1}^T p_t(-\alpha \|c_t\|^2 -\langle c_t,x_t\rangle) +\langle c_t, x_t-u\rangle+ \sum_{t\in \bad}|D_{t-1}|(\|c_t\|\|h_t\|+\alpha\|c_t\|^2),
\end{align*}
where $|D_{t-1}| = \frac{10}{\alpha\sqrt{1+\|c\|^2_{1:t-1}}}$, and the last line follows from the restrictions on $p_t$ in Algorithm \ref{alg:ogd}. The first sum in the above expression is already controlled by Lemma \ref{lem:regret-main}. For the second sum,

% If a hint is bad, then instead of bounded by $-p_t\alpha \|c_t\|^2 +(1-p_t)\langle c_t,x_t\rangle$, our loss is $-p_t\langle c_t,h_t\rangle+(1-p_t)\langle c_t,x_t\rangle=-p_t(\langle c_t,h_t\rangle-\alpha\|c_t\|^2)-p_t\alpha \|c_t\|^2+(1-p_t)\langle c_t,x_t\rangle$. That is, the gap between the loss we actually get and the loss we are planning for is at most:
% \begin{align*}
%     -p_t(\langle c_t,h_t\rangle-\alpha\|c_t\|^2)&\le D_t(\|c_t\|\|h_t\|+\alpha\|c_t\|^2)
% \end{align*}
% where recall that $D_t$ is the upper bound used to constrain $p_t$. Let $M$ be the set of indices for which  $\langle c_t,h_t\rangle \le \alpha \|c_t\|^2$. Now,  observe that:
% \begin{align*}
%     D_t &\le \frac{K}{\alpha\sqrt{1+\|c\|^2_{1:t}}}\\
%     &\le \frac{K}{\alpha\sqrt{1+\sum_{\tau \in M, \tau\le t}\|c_\tau\|^2}}
% \end{align*}
% Therefore our total regret is  at most the $\tilde O(1/\alpha)$ from the no-bad-hints analysis plus an extra:
\begin{align*}
    \sum_{t\in \bad }|D_{t-1}|(\|c_t\|\|h_t\|+\alpha\|c_t\|^2)&\le2\sum_{t\in \bad }|D_t|(\|c_t\|\|h_t\|+\alpha\|c_t\|^2)\\
    &\le 2\sum_{t\in \bad }\frac{10\|c_t\|^2}{\sqrt{1+\sum_{\tau \in \bad, \tau\le t}\|c_\tau\|^2}}+ |D_t|\|c_t\|\|h_t\|\\
    &\le 40\sqrt{\sum_{t\in \bad}  \|c_t\|^2 }+ 2\sum_{t\in \bad }|D_t|\|c_t\|\|h_t\| \\
    (\text{by Cauchy--Schwarz}) &\le 40\sqrt{\sum_{t\in \bad }  \|c_t\|^2 } +2\sqrt{\sum_{t\in \bad }\|h_t\|^2}\sqrt{\sum_{t\in \bad}  \|c_t\|^2 |D_t|^2}\\
    &\le 40\sqrt{\sum_{t\in \bad }\|c_t\|^2} +\frac{20}{\alpha}\sqrt{\sum_{t\in \bad}\|h_t\|^2}\sqrt{\log(1+\|c\|^2_{1:T})}.
    \qquad \qedhere
    \end{align*}

% So in general we get at most a $ O(\sqrt{|M| \log(T)}/\alpha)$ extra penalty for bad hints. So this is overall a $1/\sqrt{\alpha}$ worse than what we got in neurips I think, but the same as in ICML.

\end{proof}

\thmoptimistic*
\begin{proof}
The idea is to get a bound in terms of $\|c_t-h_t\|^2$.  Since $\alpha=\frac{1}{4}$, $t\in \bad$ is equivalent  to $\langle c_t, h_t\rangle \le \frac{\|c_t\|^2}{4}$. Thus if $t\in \bad$:
\begin{align*}
    \|c_t-h_t\|^2=&\|c_t\|^2-2\langle c_t,h_t\rangle  + \|h_t\|^2
    \ge \frac{\|c_t\|^2}{2} + \|h_t\|^2.
\end{align*}
Therefore, we have:
\begin{align*}
    40\sqrt{\sum_{t\in \bad} \|c_t\|^2}+80\sqrt{\sum_{t\in \bad}\|h_t\|^2\log(1+\|c\|^2_{1:T})}&\le 80(1+   \sqrt{\log(1+\|c\|^2_{1:T})})\sqrt{\sum_{t\in \bad}\|c_t-h_t\|^2}.
\end{align*}
Now, by Theorem \ref{thm:lesshintswithbad} we have:
\begin{align*}
    \E[\regret_{\cA, \alpha}(\vec{c})] &\le \frac{78+38\log(1+\|c\|^2_{1:T})}{\alpha}  + 40\sqrt{\sum_{t\in \bad }\|c_t\|^2} +\frac{20}{\alpha}\sqrt{\sum_{t\in \bad}\|h_t\|^2}\sqrt{\log(1+\|c\|^2_{1:T})}\\
    &\le \frac{78+38\log(1+\|c\|^2_{1:T})}{\alpha}+80\left(1+   \sqrt{\log(1+\|c\|^2_{1:T})}\right)\sqrt{\sum_{t\in \bad}\|c_t-h_t\|^2}.
    \quad \qedhere
    \end{align*}
\end{proof}

\section{Proofs for Section \ref{sec:lb}}

\thmlowerbounddet*
\begin{proof}
The main limitation of a deterministic algorithm $\cA$ is that even if it adapts to the costs seen so far, the adversary always knows if $\cA$ is going to make a hint query in the next step, and in steps where a query will not be made, the adversary knows which $x_t$ will be played by $\cA$.  

Using this intuition, we define the following four-dimensional instance. For convenience, let $e_0$ be a unit vector in $\R^4$, and let $S$ be the space orthogonal to $e_0$. The adversary constructs the instance iteratively, doing the following for $t=1, 2, \dots$:
\begin{enumerate}
    \item If the algorithm makes a hint query at time $t$, set $h_t = c_t = e_0$.
    \item If the algorithm does not make a hint query, then if $x_t$ is the point that will be played by the algorithm, set $c_t$ to be a unit vector in $S$ that is orthogonal to $x_t$ and to $c_1 + \dots + c_{t-1}$.  (Note that since $S$ is a three-dimensional subspace of $\R^4$, this is always feasible.)
\end{enumerate}
For convenience, define $I_t$ to be the set of indices $\le t$ in which the algorithm has asked for a hint. Then we first observe that for all $t$,
\begin{equation}
    \norm{\sum_{j \in [t] \setminus I_t} c_j}^2 = t - |I_t|. \label{eq:sum-lengths}
\end{equation}  
This is easy to see, because $c_t$ is always orthogonal to $e_0$, and thus is also orthogonal to $\sum_{j \in [t-1] \setminus I_{t-1}} c_j$. The equality~\eqref{eq:sum-lengths} then follows from the Pythagoras theorem.

Thus, suppose the algorithm makes $K$ queries in total (over the course of the $T$ steps). By assumption $K \le C\sqrt{T} < T/2$.  Then we have that
\[ \norm{\sum_{j \in [T]} c_j}^2 = K^2 + \norm{\sum_{j \in [T] \setminus I_T} c_j}^2 = K^2 + T - K. \]
Thus the optimal vector in hindsight (say $u$) achieves $\sum_{j \in [T]} \iprod{c_j, u} = - \sqrt{T - K + K^2}$. 

Let us next look at the cost of the algorithm. In every step where it makes a hint query, the best cost that $\cA$ can achieve is $-1$ (by playing $-e_0$). In the other steps, the construction ensures that the cost is $0$.  Thus the regret is at least
\[ -K + \sqrt{T - K + K^2} = \frac{ T- K}{K + \sqrt{T - K + K^2}} > \frac{T/2}{K + \sqrt{T}} \ge \frac{ \sqrt{T}}{2 (1+C)}. 
\qedhere
\]
\end{proof}

\section{Proofs for Section \ref{sec:unconstrained}}\label{app:unconstrained}

In order to prove Theorems \ref{thm:unconstrained} and \ref{thm:unconstraineddeterministic}, we first provide the following technical statement that allows us to unify much the analysis:

\begin{restatable}{lemma}{lemunconstrainedquality}\label{lem:unconstrainedquality}
Suppose that $\unconstrained$ is an unconstrained online linear optimization algorithm that outputs $w_t\in \R^d$ in response to costs $c_1,\dots,c_{t-1}\in \R^d$ satisfying $\|c_\tau\|\le 1$ for all $\tau$ and guarantees for some constants $A$ and $B$ for all $u\in \R^d$:
\begin{align*}
\regret_{\unconstrained}(u, \vec{c})
     &\le \epsilon+A\|u\|\sqrt{\sum_{t=1}^T \|c_t\|^2\log(\|u\|T/\epsilon+1)} + B\|u\|\log(\|u\|T/\epsilon+1),
\end{align*}
where $\epsilon$ is an arbitrary user-specified constant.
Further, suppose $\unconstrainedoned$ is an unconstrained online linear optimization algorithm that outputs $y_t\in \R$ in response to $g_1,\dots,g_{t-1}\in \R$ satisfying $|g_\tau|\le 1$ for all $\tau$ and guarantees for all $y_\star\in \R$:
\begin{align*}
    \sum_{t=1}^T g_t(y_t-y_\star)&\le \epsilon+A|y_\star|\sqrt{\sum_{t=1}^T g_t^2\log(|y_\star|T/\epsilon+1)} + B|y_\star|\log(|y_\star|T/\epsilon+1).
\end{align*}
Finally, suppose also that $\E\left[\sum_{t=1}^T \mathbbm{1}_{t}|\langle c_t, h_t\rangle|\right]\ge M\sqrt{1+\|c\|^2_{1:T}}-N$ and $\E\left[\sum_{t\in \bad} \mathbbm{1}_{t} |\langle c_t, h_t\rangle|\right]\le H$ and $\E\left[\sum_{t=1}^T \mathbbm{1}_{t}  \langle c_t, h_t\rangle^2\right] \le F\sqrt{1+\|c\|^2_{1:T}}$ for some constant $M,N, H, F$. Then both the deterministic and randomized version of Algorithm \ref{alg:unconstrained} guarantee:
\begin{align*}
    \E\left[ \regret_{\unconstrained}(u, \vec{c})
 \right]&\le  2\epsilon + B\|u\|\log(\|u\|T/\epsilon+1) + \frac{4A\|u\|(H+N)\sqrt{\log(\|u\|T/\epsilon+1)}}{M}\\
    &\quad\quad + \frac{2AB\|u\|\sqrt{\log(\|u\|T/\epsilon+1)}\log(2A\|u\|T\sqrt{\log(\|u\|T/\epsilon+1)}/(M\epsilon)+1)}{M}\\
    &\quad\quad + \frac{2A^3F\|u\|\sqrt{\log(\|u\|T/\epsilon+1)\log(2A\|u\|T\sqrt{\log(\|u\|T/\epsilon+1)}/(M\epsilon)+1)}}{M^2}.
\end{align*}
\end{restatable}

\begin{proof}[Proof of Lemma \ref{lem:unconstrainedquality}]

Some algebraic manipulation of the regret definition yields:
\begin{align*}
    & \E[ \regret_{\unconstrained}(u, \vec{c})
 ]\le \E\left[\inf_{y_\star} \sum_{t=1}^T \langle c_t,w_t - u\rangle - y_\star \sum_{t=1}^T \mathbbm{1}_{t} \langle h_t,c_t\rangle - \sum_{t=1}^T \mathbbm{1}_{t} \langle h_t,c_t\rangle(y_t-y_\star)\right]\\
    &\le \E\left[\inf_{y_\star \ge 0} \sum_{t=1}^T \langle c_t,w_t - u\rangle - y_\star \sum_{t=1}^T \mathbbm{1}_{t} |\langle h_t,c_t\rangle| + 2y_\star \sum_{t\in \bad} \mathbbm{1}_{t} |\langle h_t,c_t\rangle| - \sum_{t=1}^T \mathbbm{1}_{t} \langle h_t,c_t\rangle(y_t-y_\star)\right].
\end{align*}
Now using the hypothesized bounds we have
\begin{align*}
    & \E\left[ \regret_{\unconstrained}(u, \vec{c}) \right] 
    \le \E\left[\inf_{y_\star \ge 0} \sum_{t=1}^T \langle c_t,w_t - u\rangle - y_\star M\sqrt{1+\|c\|^2_{1:T}} + 2y_\star H +y_\star N  - \sum_{t=1}^T \mathbbm{1}_{t} \langle h_t,c_t\rangle(y_t-y_\star)\right]\\
    &\le \inf_{y_\star\ge 0}\E\left[2\epsilon + A\|u\|\sqrt{\sum_{t=1}^T \|c_t\|^2\log(\|u\|T/\epsilon+1)} + B\|u\|\log(\|u\|T/\epsilon+1)\right.\\
    &\quad\quad - y_\star M\sqrt{1+\|c\|^2_{1:T}} + 2y_\star H +y_\star N
    \left.+A y_\star\sqrt{\sum_{t=1}^t g_t^2 \log(y_\star T/\epsilon+1)} + By_\star\log(y_\star T/\epsilon+1)\right]
    \intertext{using Jensen inequality,}
    &\le \inf_{y_\star\ge 0} 2\epsilon + A\|u\|\sqrt{\sum_{t=1}^T \|c_t\|^2\log(\|u\|T/\epsilon+1)} + B\|u\|\log(\|u\|T/\epsilon+1)
    - y_\star M\sqrt{1+\|c\|^2_{1:T}} \\
    &\quad\quad + 2y_\star H +y_\star N
    +A y_\star\sqrt{\E\left[\sum_{t=1}^t \mathbbm{1}_{t}\langle c_t, h_t\rangle^2\right] \log(y_\star T/\epsilon+1)} + By_\star\log(y_\star T/\epsilon+1)\\
    &\le \inf_{y_\star\ge 0} 2\epsilon + A\|u\|\sqrt{\sum_{t=1}^T \|c_t\|^2\log(\|u\|T/\epsilon+1)} + B\|u\|\log(\|u\|T/\epsilon+1)
    - y_\star M\sqrt{1+\|c\|^2_{1:T}} \\
    & \quad\quad + 2y_\star H +y_\star N
    +A y_\star\sqrt{F\sqrt{1+\|c\|^2_{1:T}} \log(y_\star T/\epsilon+1)} + By_\star\log(y_\star T/\epsilon+1)
    \intertext{with a little rearrangement,}
    &\le \inf_{y_\star\ge 0} 2\epsilon + A\|u\|\sqrt{\sum_{t=1}^T \|c_t\|^2\log(\|u\|T/\epsilon+1)} + B\|u\|\log(\|u\|T/\epsilon+1) - \frac{y_\star}{2} M\sqrt{\|c\|^2_{1:T}}\\
    &\quad\quad + 2y_\star H +y_\star N + By_\star\log(y_\star T/\epsilon+1)\\
    &\quad\quad+A y_\star\sqrt{F\sqrt{1+\|c\|^2_{1:T}} \log(y_\star T/\epsilon+1)} - \frac{y_\star}{2} M\sqrt{1+\|c\|^2_{1:T}}\\
    &\le \inf_{y_\star\ge 0} 2\epsilon + A\|u\|\sqrt{\sum_{t=1}^T \|c_t\|^2\log(\|u\|T/\epsilon+1)} + B\|u\|\log(\|u\|T/\epsilon+1) - \frac{y_\star}{2} M\sqrt{\|c\|^2_{1:T}}\\
    &\quad\quad + 2y_\star H +y_\star N + By_\star\log(y_\star T/\epsilon+1) 
    +\sup_{X} A y_\star\sqrt{FX\log(y_\star T/\epsilon+1)} - \frac{y_\star}{2} MX\\
    &\le \inf_{y_\star\ge 0} 2\epsilon + A\|u\|\sqrt{\sum_{t=1}^T \|c_t\|^2\log(\|u\|T/\epsilon+1)} + B\|u\|\log(\|u\|T/\epsilon+1) - \frac{y_\star}{2} M\sqrt{\|c\|^2_{1:T}}\\
    &\quad\quad + 2y_\star H +y_\star N + By_\star\log(y_\star T/\epsilon+1)
    +\frac{y_\star A^2 F\log(y_\star T/\epsilon+1)}{2 M}.
\end{align*}
Now, we set 
\begin{align*}
    y_\star = \frac{2A\|u\|\sqrt{\log(\|u\|T/\epsilon+1)}}{M}.
\end{align*}
This yields
\begin{align*}
    & \E[ \regret_{\unconstrained}(u, \vec{c}) ] \\
    &\le 2\epsilon + B\|u\|\log(\|u\|T/\epsilon+1) + 2y_\star H +y_\star N + By_\star\log(y_\star T/\epsilon+1)
    +\frac{y_\star A^2 F\log(y_\star T/\epsilon+1)}{2 M}\\
    &\le 2\epsilon + B\|u\|\log(\|u\|T/\epsilon+1) + \frac{4A\|u\|(H+N)\sqrt{\log(\|u\|T/\epsilon+1)}}{M}\\
    &\quad\quad + \frac{2AB\|u\|\sqrt{\log(\|u\|T/\epsilon+1)}\log(2A\|u\|T\sqrt{\log(\|u\|T/\epsilon+1)}/(M\epsilon)+1)}{M}\\
    &\quad\quad + \frac{2A^3F\|u\|\sqrt{\log(\|u\|T/\epsilon+1)\log(2A\|u\|T\sqrt{\log(\|u\|T/\epsilon+1)}/(M\epsilon)+1)}}{M^2}.
\qedhere
\end{align*}
\end{proof}

Now, to prove Theorem \ref{thm:unconstrained}, it suffices to instantiate the Lemma. We restate the Theorem below for convenience:

\thmunconstrained*
\begin{proof}%[Proof of Theorem \ref{thm:unconstrained}]
Define
\begin{align*}
    p_t = \min\left(1,\frac{K}{\alpha \sqrt{1+\|c\|^2_{1:t}}}\right),
\end{align*}
so that in the randomized version of Algorithm \ref{alg:unconstrained}, at round $t$, we ask for a hint with probability $p_{t-1}$. Clearly, the expected query cost is:
\begin{align*}
    \E\left[\sum_{t=1}^T \mathbbm{1}_{t}\langle c_t, h_t\rangle\right] = \sum_{t=1}^T \alpha p_{t-1} \|c_t\|^2
    \le K\sum_{t=1}^T\frac{\|c_t\|^2}{\sqrt{\|c\|^2_{1:t}}} 
    \le 2K\sqrt{\|c\|^2_{1:T}}.
\end{align*}

Now, to bound the regret we consider two cases. First, if $1+\|c\|^2_{1:T}\le \frac{K^2}{\alpha^2}$, then we have:
\begin{align*}
    \E[ \regret_{\unconstrained}(u, \vec{c}) ]&\le \E\left[\sum_{t=1}^T \langle c_t,  w_t - u\rangle - \sum_{t=1}^T \mathbbm{1}_{t}\langle c_t, h_t\rangle y_t\right ]
    \le \E\left[\sum_{t=1}^T \langle c_t,  w_t - u\rangle + \sum_{t=1}^Tg_t(y_t-0) \right]\\
    &\le 2\epsilon+A\|u\|\sqrt{\sum_{t=1}^T \|c_t\|^2\log(\|u\|T/\epsilon+1)} + B\|u\|\log(\|u\|T/\epsilon+1)\\
    &\le 2\epsilon+\frac{A\|u\|K\sqrt{\log(\|u\|T/\epsilon+1)}}{\alpha} + B\|u\|\log(\|u\|T/\epsilon+1),
\end{align*}
and so the result follows. Thus, we may assume $1+\|c\|^2_{1:T}> \frac{K^2}{\alpha^2}$. In this case, we will calculate values for $M$, $H$, and $F$ to use in tandem with Lemma \ref{lem:unconstrainedquality}. First,
\begin{align*}
    \E\left[\sum_{t=1}^T \mathbbm{1}_{t}\langle c_t, h_t\rangle^2\right] \le  \sum_{t=1}^T  p_{t-1} \|c_t\|^2
     \le \frac{K}{\alpha}\sum_{t=1}^T\frac{\|c_t\|^2}{\sqrt{\|c\|^2_{1:t}}}
     \le \frac{2K}{\alpha}\sqrt{1+\|c\|^2_{1:T}}.
\end{align*}
So that we may take $F=\frac{2K}{\alpha}$. Next, note that $p_T = \frac{K}{\alpha \sqrt{1+\|c\|^2_{1:T}}}$ by our casework assumption. Therefore:
\begin{align*}
    -\alpha p_T \|c\|^2_{1:T} \le \alpha - \alpha p_T(1+\|c\|^2_{1:T})
    \le \alpha -K\sqrt{\|c\|^2_{1:T}},
\end{align*}
so that we may take $M=K$ and $N=\alpha$. Finally,

\begin{align*}
    \sum_{t\in \B} p_t |\langle c_t, h_t\rangle|
    \le K\sum_{t\in \B} \frac{\|c_t\|\|h_t\|}{\alpha \sqrt{\|c_t\|^2_{1:t}}} 
    \le\frac{K}{\alpha} \sqrt{\sum_{t\in\B} \frac{\|c_t\|^2}{\|c_t\|^2_{1:t}}\sum_{t\in \B}\|h_t\|^2}
    \le \frac{K}{\alpha}\sqrt{\sum_{t\in \B}\|h_t\|^2\log(1+\|c\|^2_{1:T})},
\end{align*}
so that we may take $H=\frac{K}{\alpha}\sqrt{\sum_{t\in \B}\|h_t\|^2\log(1+\|c\|^2_{1:T})}$. Then Lemma \ref{lem:unconstrainedquality} implies

\begin{align*}
    \E[ \regret_{\unconstrained}(u, \vec{c}) ] &\le 2\epsilon + B\|u\|\log(\|u\|T/\epsilon+1) + \frac{4A\|u\|(H+\alpha)\sqrt{\log(\|u\|T/\epsilon+1)}}{M}\\
    &\quad\quad + \frac{2AB\|u\|\sqrt{\log(\|u\|T/\epsilon+1)}\log(2A\|u\|T\sqrt{\log(\|u\|T/\epsilon+1)}/(M\epsilon)+1)}{M}\\
    &\quad\quad + \frac{2A^3F\|u\|\sqrt{\log(\|u\|T/\epsilon+1)\log(2A\|u\|T\sqrt{\log(\|u\|T/\epsilon+1)}/(M\epsilon)+1)}}{M^2}\\
    &\le 2\epsilon + B\|u\|\log(\|u\|T/\epsilon+1)+ \frac{4A\|u\|\sqrt{\log(\|u\|T/\epsilon+1)\sum_{t\in \B}\|h_t\|^2\log(1+\|c\|^2_{1:T})}}{\alpha}\\
    &\quad\quad+ \frac{4A\|u\|\alpha\sqrt{\log(\|u\|T/\epsilon+1)}}{K}\\
    &\quad\quad + \frac{2AB\|u\|\sqrt{\log(\|u\|T/\epsilon+1)}\log(2A\|u\|T\sqrt{\log(\|u\|T/\epsilon+1)}/(K\epsilon)+1)}{K}\\
    &\quad\quad + \frac{2A^3\|u\|\sqrt{\log(\|u\|T/\epsilon+1)\log(2A\|u\|T\sqrt{\log(\|u\|T/\epsilon+1)}/(K\epsilon)+1)}}{K\alpha}.
\end{align*}
Simplifying the expression yields
\begin{align*}
    & \E[ \regret_{\unconstrained}(u, \vec{c}) ] \\
    &\le 2\epsilon+\tilde O\left(\frac{\|u\|(\frac{\log(\|u\|T/\epsilon)^{3/2}\log\log(T\|u\|/\epsilon)}{K}+\sqrt{\log(\|u\|T/\epsilon)\sum_{t\in \B}\|h_t\|^2\log(1+\|c\|^2_{1:T})})}{\alpha}\right).
    \qedhere
    \end{align*}
\end{proof}

\subsection{Deterministic version}

Before providing the proof of Theorem \ref{thm:unconstraineddeterministic}, we need the following auxiliary statement.

\begin{lemma}\label{lem:det_unc_enough_hints}
Suppose $\B=\emptyset$. Then for all $t$, the deterministic version of Algorithm \ref{alg:unconstrained} guarantees:
\begin{align*}
    \sqrt{\|c\|^2_{1:T-1}} -K-1-\frac{K}{2\alpha} \le \sum_{t=1}^T \mathbbm{1}_{t} \langle c_t, h_t\rangle \le K\sqrt{1+\|c\|^2_{1:T-1}}.
\end{align*}
\end{lemma}
\begin{proof}
Define $Z_t = 1+\sum_{t=1}^T \mathbbm{1}_{t} \langle c_t, h_t\rangle$ with $Z_0=1$.
We will instead prove the equivalent bound:
\begin{align*}
    K\sqrt{\|c\|^2_{1:T-1}} -K- \frac{K}{2\alpha} \le Z_T \le 1+K\sqrt{1+\|c\|^2_{1:T-1}}.
\end{align*}
The upper bound is immediate from the definition of $Z_T$ and the fact that $\langle c_t,h_t\rangle \le 1$.
For the lower bound, we will prove a slightly different statement that we will later show implies the desired result: 
\begin{align*}
    \mbox{ for all $t\ge 0$, } Z_t \ge K\sqrt{1+\|c\|^2_{1:t}} -K\sum_{t'\le t| \sqrt{\|c\|^2_{1:t'}}\le \frac{1}{2\alpha}} \frac{\|c_{t'}\|^2}{2\sqrt{\|c\|^2_{1:t'}}}.
\end{align*}
We proceed by induction. The base case for $t=0$ is clear from definition of $Z_t$. Suppose the statement holds for some $t$. Then consider two cases, either $Z_t<K\sqrt{1+\|c\|^2_{1:t}}$ or not. If $Z_t\ge K\sqrt{1+\|c\|^2_{1:t}}$, then $Z_{t+1}=Z_t\ge K\sqrt{1+\|c\|^2_{1:t}}\ge K\sqrt{1+\|c\|^2_{1:t+1}}-K$ and so the statement holds. Alternatively, suppose $Z_t<K\sqrt{1+\|c\|^2_{1:t}}$. Then:

\begin{align*}
    Z_{t+1} &= Z_t + \langle c_{t+1},h_{t+1}\rangle\\
    &\ge K\sqrt{1+\|c\|^2_{1:t}} - K-\sum_{t'\le t| \sqrt{\|c\|^2_{1:t}}\le \frac{K}{2\alpha}} \frac{\|c_{t'}\|^2}{2\sqrt{\|c\|^2_{1:t'}}} + \alpha \|c_{t+1}\|^2\\
    &\ge K\sqrt{1+\|c\|^2_{1:t+1}} - \frac{K\|c_{t+1}\|^2}{2\sqrt{1+\|c\|^2_{1:t}}} - K-\sum_{t'\le t| \sqrt{\|c\|^2_{1:t}}\le \frac{K}{2\alpha}} \frac{\|c_{t'}\|^2}{2\sqrt{\|c\|^2_{1:t'}}} + \alpha \|c_{t+1}\|^2\\
    &\ge K\sqrt{1+\|c\|^2_{1:t+1}} - \frac{K\|c_{t+1}\|^2}{2\sqrt{\|c\|^2_{1:t+1}}} - K-\sum_{t'\le t| \sqrt{\|c\|^2_{1:t'}}\le \frac{K}{2\alpha}} \frac{\|c_{t'}\|^2}{2\sqrt{\|c\|^2_{1:t'}}} + \alpha \|c_{t+1}\|^2\\
    &\ge \sqrt{1+\|c\|^2_{1:t+1}}- K-\sum_{t'\le t+1| \sqrt{\|c\|^2_{1:t'}}\le \frac{K}{2\alpha}} \frac{\|c_{t'}\|^2}{2\sqrt{\|c\|^2_{1:t'}}},
\end{align*}
so that the induction is complete.

Finally, observe that if $\tau$ is the largest index such that $\sqrt{\|c\|^2_{1:t}}\le \frac{K}{2\alpha}$, then
\begin{align*}
    \sum_{t'\le t+1| \sqrt{\|c\|^2_{1:t'}}\le \frac{K}{2\alpha}} \frac{\|c_{t'}\|^2}{2\sqrt{\|c\|^2_{1:t'}}}
    \le \sum_{t'=1}^\tau \frac{\|c_{t'}\|^2}{2\sqrt{\|c\|^2_{1:t'}}}
    \le \sqrt{\|c\|^2_{1:\tau}}
    \le \frac{K}{2\alpha}.
    \qquad \qedhere
\end{align*}
\end{proof}

Now we can prove Theorem \ref{thm:unconstraineddeterministic}:

\thmunconstraineddeterministic*
\begin{proof}%[Proof of Theorem \ref{thm:unconstraineddeterministic}]
From Lemma \ref{lem:det_unc_enough_hints} we have that the query cost is at most $K\sqrt{\|c\|^2_{1:T}}$. To bound the regret, we will appeal to Lemma \ref{lem:unconstrainedquality}, which requires finding values for $M,N,H,F$. First, again by Lemma \ref{lem:det_unc_enough_hints}, we have: 
\begin{align*}
    K\sqrt{1+\|c\|^2_{1:T}} -3K-1-\frac{K}{2\alpha}\le K\sqrt{\|c\|^2_{1:T-1}} - K-1-\frac{K}{2\alpha}\le \sum_{t=1}^T \mathbbm{1}_{t}\langle c_t, h_t\rangle.
\end{align*}
So that we may set $M=K$ and $N=3K+1+\frac{K}{2\alpha}$. Next, since $\B=\emptyset$, $H=0$. Finally, since all hints are $\alpha$-good, we have
\begin{align*}
    \sum_{t=1}^T \mathbbm{1}_{t}\langle c_t, h_t\rangle^2&\le \sum_{t=1}^T  \mathbbm{1}_{t}\langle c_t,h_t\rangle\le K\sqrt{\|c\|^2_{1:T}},
\end{align*}
so that we may take $F=K$. Therefore, noticing that the expected regret is the actual regret since the algorithm is deterministic, we have

\begin{align*}
    \regret_{\unconstrained}(u, \vec{c}) &\le  2\epsilon + B\|u\|\log(\|u\|T/\epsilon+1) + \frac{4A\|u\|(H+N)\sqrt{\log(\|u\|T/\epsilon+1)}}{M}\\
    &\quad\quad + \frac{2AB\|u\|\sqrt{\log(\|u\|T/\epsilon+1)}\log(2A\|u\|T\sqrt{\log(\|u\|T/\epsilon+1)}/(M\epsilon)+1)}{M}\\
    &\quad\quad + \frac{2A^3F\|u\|\sqrt{\log(\|u\|T/\epsilon+1)\log(2A\|u\|T\sqrt{\log(\|u\|T/\epsilon+1)}/(M\epsilon)+1)}}{M^2}\\
    &\le 2\epsilon + B\|u\|\log(\|u\|T/\epsilon+1) + 4A\|u\|\left(\frac{4}{\alpha} + \frac{1}{K}\right)\sqrt{\log(\|u\|T/\epsilon+1)}\\
    &\quad\quad + \frac{2AB\|u\|\sqrt{\log(\|u\|T/\epsilon+1)}\log(2A\|u\|T\sqrt{\log(\|u\|T/\epsilon+1)}/(K\epsilon)+1)}{K}\\
    &\quad\quad + \frac{2A^3\|u\|\sqrt{\log(\|u\|T/\epsilon+1)\log(2A\|u\|T\sqrt{\log(\|u\|T/\epsilon+1)}/(K\epsilon)+1)}}{K}\\
    &\le 2\epsilon + O\left(\frac{\|u\|\sqrt{\log(\|u\|T/\epsilon+1)}}{\alpha} + \frac{\|u\|\log^{3/2}(\|u\|T/\epsilon)\log\log(\|u\|T/\epsilon)}{K}\right).
    \qedhere
\end{align*}

\end{proof}

\end{document}